%% file: main.tex
\documentclass[journal,twoside]{cls/IEEEtran}

\usepackage[table,usenames,dvipsnames]{xcolor}
\usepackage{cite}
\usepackage{amsmath,amssymb,amsfonts,amsthm,amscd,dsfont}
\usepackage{accents} 
\usepackage{algorithm,listings}
\usepackage[noend]{algpseudocode}
\usepackage{graphicx,tabularx,adjustbox}
\usepackage[font={small}]{caption}
\usepackage{subcaption}
\usepackage{paralist}

\let\appendices\relax
\usepackage[titletoc,title]{appendix}
\usepackage{balance}
\usepackage[breaklinks=true, colorlinks, bookmarks=true, citecolor=Black, urlcolor=Violet,linkcolor=Black]{hyperref}

\input{tex/sym.tex}

\title{\LARGE \bf Semantic OcTree Mapping and Shannon Mutual Information Computation for Robot Exploration%
}

\author{Arash~Asgharivaskasi,~\IEEEmembership{Student Member,~IEEE,} Nikolay~Atanasov,~\IEEEmembership{Member,~IEEE}
\thanks{We gratefully acknowledge support from ARL DCIST CRA W911NF-17-2-0181 and NSF FRR CAREER 2045945. The authors are with the Department of Electrical and Computer Engineering, University of California San Diego, CA 92093, USA {\tt\small \{aasghari,natanasov\}@eng.ucsd.edu}.}}

\begin{document}

\markboth{IEEE Transactions on Robotics,~Vol.~X, No.~X, Month~20XX}%
{Asgharivaskasi and Atanasov: Semantic OcTree Mapping and Shannon Mutual Information Computation for Robot Exploration}
%

\maketitle

\begin{abstract}
Autonomous robot operation in unstructured and unknown environments requires efficient techniques for mapping and exploration using streaming range and visual observations. Information-based exploration techniques, such as Cauchy-Schwarz quadratic mutual information (CSQMI) and fast Shannon mutual information (FSMI), have successfully achieved active binary occupancy mapping with range measurements. However, as we envision robots performing complex tasks specified with semantically meaningful concepts, it is necessary to capture semantics in the measurements, map representation, and exploration objective. This work presents Semantic octree mapping and Shannon Mutual Information (SSMI) computation for robot exploration. We develop a Bayesian multi-class mapping algorithm based on an octree data structure, where each voxel maintains a categorical distribution over semantic classes. We derive a closed-form efficiently-computable lower bound of the Shannon mutual information between a multi-class octomap and a set of range-category measurements using semantic run-length encoding of the sensor rays. The bound allows rapid evaluation of many potential robot trajectories for autonomous exploration and mapping. We compare our method against state-of-the-art exploration techniques and apply it in a variety of simulated and real-world experiments.
\end{abstract}

\begin{IEEEkeywords}
View Planning for SLAM, Reactive and Sensor-Based Planning, Vision-Based Navigation.
\end{IEEEkeywords}

\IEEEpeerreviewmaketitle


\input{tex/introduction}
\input{tex/relatedWork}
\input{tex/problemStatement}
\input{tex/perception}

\input{tex/informativePlanning}

\input{tex/octomap}
\input{tex/experiments}

\input{tex/conclusion}

\appendices
\input{tex/appendix.tex}



\ifCLASSOPTIONcaptionsoff
  \newpage
\fi


{\small
\bibliographystyle{cls/IEEEtran}
\bibliography{bib/IEEEexample.bib}
}

%








\begin{IEEEbiography}
[{\includegraphics[width=1in,height=1.25in,clip,keepaspectratio]{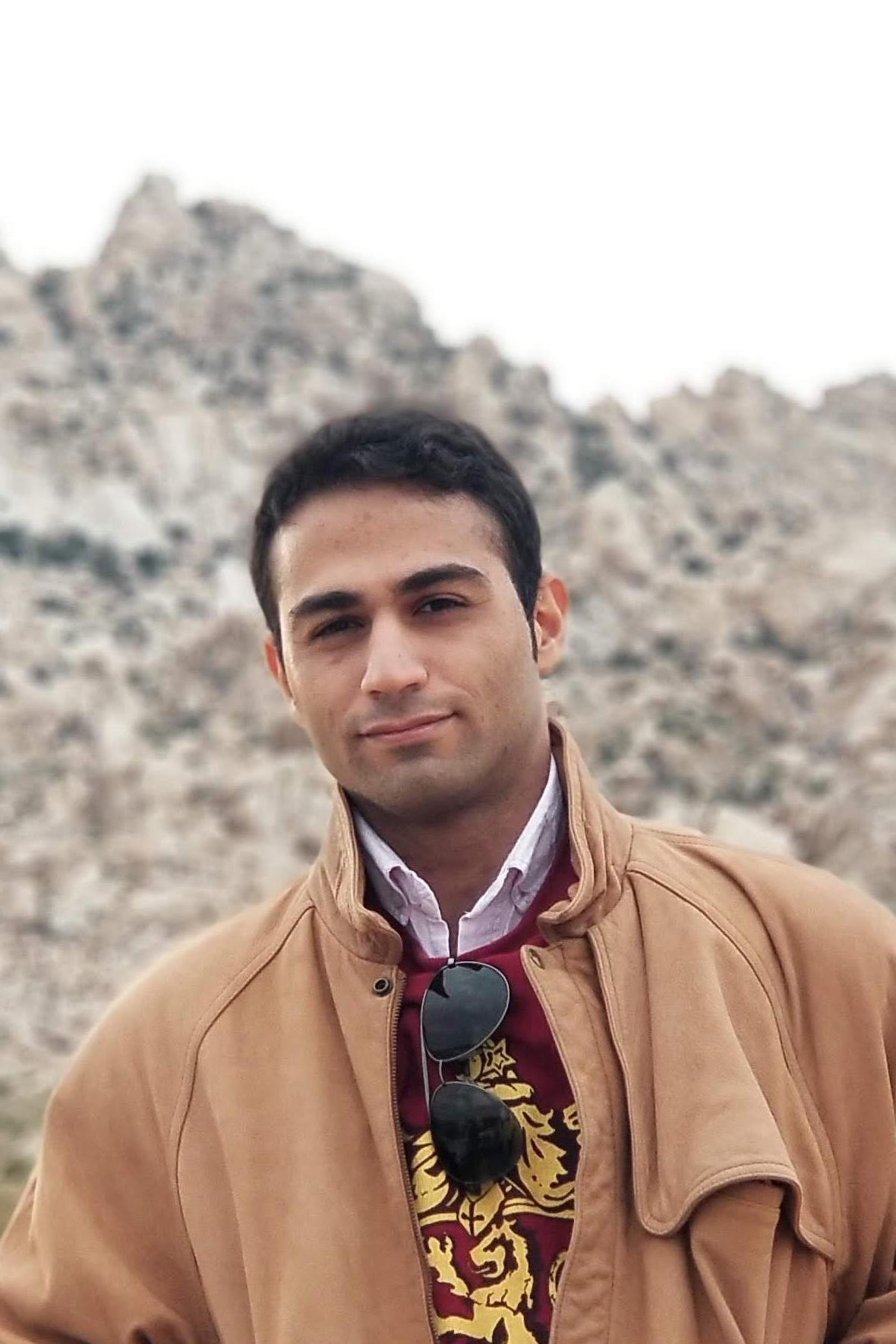}}]{Arash Asgharivaskasi}
(S'22) is a Ph.D. student of Electrical and Computer Engineering at the University of California San Diego, La Jolla, CA. He obtained a B.S. degree in Electrical Engineering from Sharif University of Technology, Tehran, Iran, and an M.S. degree in Electrical and Computer Engineering from the University of California San Diego, La Jolla, CA, USA. His research focuses on active information acquisition using mobile robots with applications to mapping, security, and environmental monitoring.
\end{IEEEbiography}
\vspace{-150 mm}
\begin{IEEEbiography}
[{\includegraphics[width=1in,height=1.25in,clip,keepaspectratio]{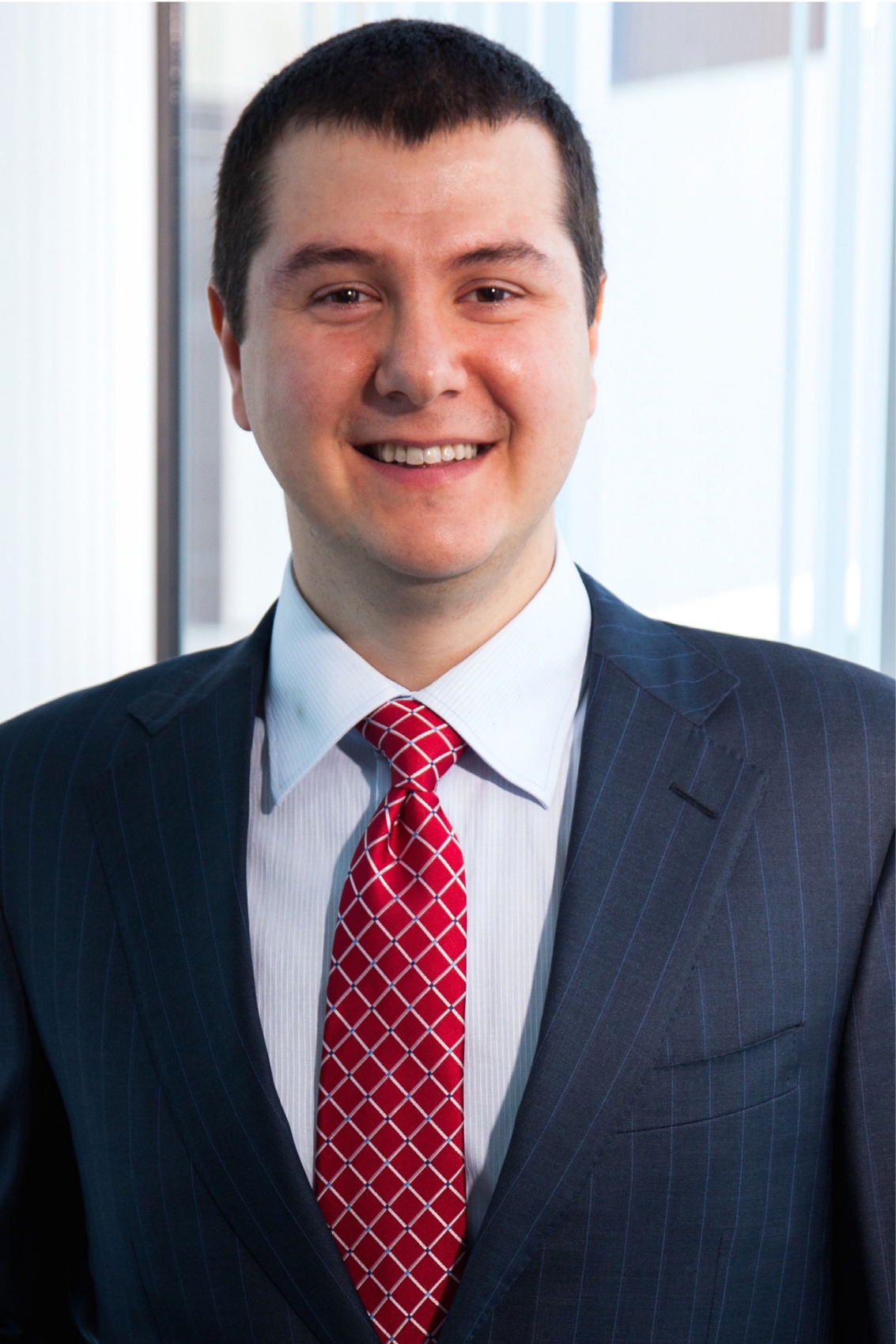}}]{Nikolay Atanasov}
(S'07-M'16) is an Assistant Professor of Electrical and Computer Engineering at the University of California San Diego, La Jolla, CA, USA. He obtained a B.S. degree in Electrical Engineering from Trinity College, Hartford, CT, USA, in 2008 and M.S. and Ph.D. degrees in Electrical and Systems Engineering from the University of Pennsylvania, Philadelphia, PA, USA, in 2012 and 2015, respectively. His research focuses on robotics, control theory, and machine learning, applied to active sensing using ground and aerial robots. He works on probabilistic perception models that unify geometry and semantics and on optimal control and reinforcement learning approaches for minimizing uncertainty in these models. Dr. Atanasov's work has been recognized by the Joseph and Rosaline Wolf award for the best Ph.D. dissertation in Electrical and Systems Engineering at the University of Pennsylvania in 2015, the best conference paper award at the International Conference on Robotics and Automation in 2017, and the NSF CAREER award in 2021.
\end{IEEEbiography}
\end{document}

%% file: tex/sym.tex
\newcommand{\calE}{{\cal E}}
\newcommand{\calF}{{\cal F}}

\newcommand{\calI}{{\cal I}}

\newcommand{\calK}{{\cal K}}

\newcommand{\calN}{{\cal N}}

\newcommand{\calR}{{\cal R}}

\newcommand{\calU}{{\cal U}}

\newcommand{\calZ}{{\cal Z}}

\newcommand{\bfe}{\mathbf{e}}

\newcommand{\bfh}{\mathbf{h}}

\newcommand{\bfl}{\mathbf{l}}
\newcommand{\bfm}{\mathbf{m}}
\newcommand{\bfn}{\mathbf{n}}

\newcommand{\bfp}{\mathbf{p}}

\newcommand{\bfu}{\mathbf{u}}
\newcommand{\bfv}{\mathbf{v}}

\newcommand{\bfz}{\mathbf{z}}

\newcommand{\bfeta}{\boldsymbol{\eta}}

\newcommand{\bfphi}{\boldsymbol{\phi}}
\newcommand{\bfchi}{\boldsymbol{\chi}}
\newcommand{\bfpsi}{\boldsymbol{\psi}}
\newcommand{\bfomega}{\boldsymbol{\omega}}

\newcommand{\bfE}{\mathbf{E}}

\newcommand{\bfR}{\mathbf{R}}

\newcommand{\bfX}{\mathbf{X}}

\newcommand{\bbE}{\mathbb{E}}

\newcommand{\bbR}{\mathbb{R}}

\newcommand{\prl}[1]{\left(#1\right)}

\newcommand{\crl}[1]{\left\{#1\right\}}

\newcommand{\scaleMathLine}[2][1]{\resizebox{#1\linewidth}{!}{$\displaystyle{#2}$}}

\def\negquad{\mkern-18mu}

\newcommand\BibTeX{{\rmfamily B\kern-.05em \textsc{i\kern-.025em b}\kern-.08em
T\kern-.1667em\lower.7ex\hbox{E}\kern-.125emX}}

\algnewcommand{\LineComment}[1]{\Statex \(\triangleright\) #1}

\newtheorem{proposition}{Proposition}
\newtheorem{corollary}{Corollary}

\theoremstyle{definition}
\newtheorem{definition}{Definition}

\newtheorem*{problem}{Problem}
\theoremstyle{remark}

\newsavebox{\tempbox}

%% file: tex/introduction.tex
\section{Introduction}
\label{sec:introduction}

Accurate modeling, real-time understanding, and efficient storage of a robot's environment are key capabilities for autonomous operation. Occupancy grid mapping \cite{occ_mapping_1, occ_mapping_2} is a widely used, simple, yet effective, technique for distinguishing between traversable and occupied space surrounding a mobile robot. OctoMap \cite{octomap} is an extension of occupancy grid mapping, introducing adaptive resolution to improve the memory usage and computational cost of generating 3-D maps of large environments. Delegating increasingly sophisticated tasks to autonomous robots requires augmenting traditional geometric models with semantic information about the context and object-level structure of the environment \cite{semantic_1, semantic_2, semantic_3}. Robots also are increasingly expected to operate in unknown environments, with little to no prior information, in applications such as disaster response, environmental monitoring, and reconnaissance. This calls for algorithms allowing robots to autonomously explore unknown environments and construct low-uncertainty metric-semantic maps in real-time, while taking collision and visibility constraints into account.

\begin{figure}[t]
  \centering
  \includegraphics[width=0.9\linewidth]{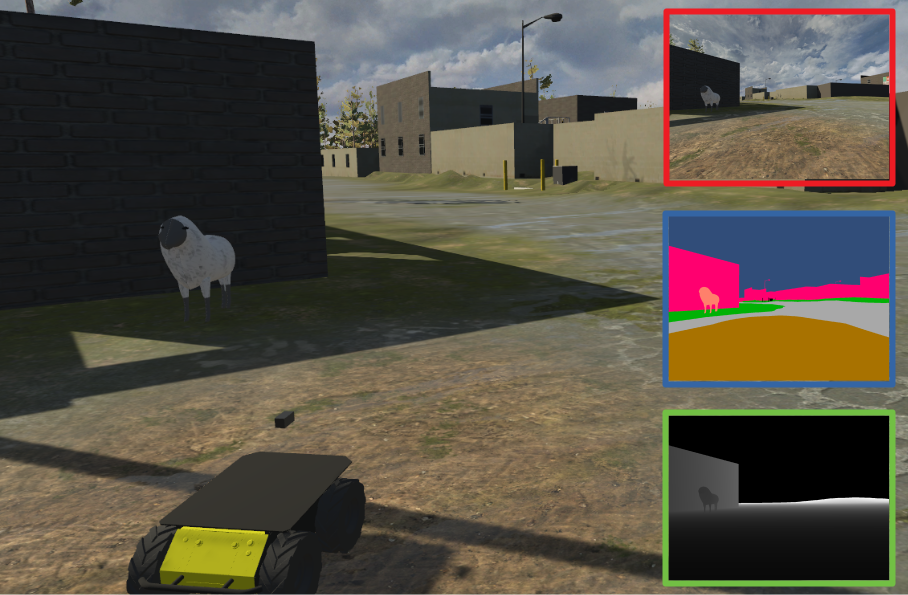}
  \caption{A robot autonomously explores an unknown environment using an RGBD sensor and a semantic segmentation algorithm.}
    \label{fig:prob_state}
\end{figure}

This paper considers the active metric-semantic mapping problem, requiring a robot to explore and map an unknown environment, relying on streaming distance and object category observations, e.g., generated by semantic segmentation over RGBD images \cite{bonnet}. See Fig.~\ref{fig:prob_state} as an illustration of the sensor data available for mapping and exploration. We extend information-theoretic active mapping techniques \cite{julian, csqmi, fsmi} from binary to multi-class environment representations. Our approach introduces a Bayesian multi-class mapping procedure which maintains a probability distribution over semantic categories and updates it via a probabilistic range-category perception model. Our main \textbf{contributions} are:
\begin{enumerate}
  \item a Bayesian multi-class octree mapping approach,
  \item a closed-form efficiently-computable lower bound for the Shannon mutual information between a multi-class octree map and a set of range-category measurements,
  \item efficient C++ implementation achieving real-time high-accuracy performance onboard physical robot systems in 3-D real-world experiments\footnote{Open-source software and videos supplementing this paper are available at \url{https://arashasgharivaskasi-bc.github.io/SSMI_webpage}.}.
\end{enumerate}
Unlike a uniform-resolution grid map, our semantic OctoMap enables efficient mutual information evaluation via run-length encoding of the range-category observations. As a result, the informativeness of many potential robot trajectories may be evaluated to (re-)select one that leads to the best trade-off between map uncertainty reduction and motion cost. Unlike traditional class-agnostic exploration methods, our model and information measure capture the uncertainty of different semantic classes, leading to faster and more accurate exploration. The proposed approach relies on general range and class measurements and general pose kinematics, making it suitable for either ground or aerial robots, equipped with either camera or LiDAR sensors, exploring either indoor or outdoor environments.

We name our method \textit{Semantic octree mapping and Shannon Mutual Information} (SSMI) computation for robot exploration. This paper is an extended version of our previous conference paper \cite{ssmi_icra} enabling its application in large-scale 3-D environments. Direct application of the regular-grid method in the conference version to 3-D environments faces several challenges, including memory inefficiency, because 3-D environments are predominantly made up of free or unknown space, and information computation inefficiency due to na{\"i}ve summation over all voxels visited by sensor rays. We introduce (a) an octree representation for Bayesian multi-class mapping, performing octree ray-tracing and leaf node merging using semantic category distributions, (b) semantic run-length encoding (SRLE) for sensor rays which enables efficient mutual information computation between a multi-class octree and range-category measurements, and (c) a comprehensive set of experiments and comparisons with state-of-the-art techniques in 2-D and 3-D environments. We demonstrate the effectiveness of SSMI in real-world mapping and autonomous exploration experiments performed in real-time onboard different mobile robot platforms operating in indoor and outdoor environments.

%% file: tex/relatedWork.tex
\section{Related Work}
\label{sec:related_work}

Frontier-based exploration \cite{frontier} is a seminal work that highlighted the utility of autonomous exploration and active mapping. It inspired methods \cite{geo_exp_1, geo_exp_2} that rely on geometric features, such as the boundaries between free and unknown space (frontiers) and the volume that would be revealed by new sensor observations. Due to their intuitive formulation and low computational requirements, geometry-based exploration methods continue to be widely employed in active perception. Recent works include semantics-assisted indoor exploration \cite{geo_exp_3}, hex-decomposition-based coverage planning \cite{geo_exp_6}, and Laplace potential fields for safe outdoor exploration \cite{geo_exp_5}. More related to our work, receding-horizon ``next-best-view'' planning \cite{tare_1} presents an active octree occupancy mapping method which executes trajectories built from a random tree whose quality is determined by the amount of unmapped space that can be explored. Similarly, the graph-based exploration methods of \cite{tare_2} and \cite{tare_3} use local random trees in free space to sample candidate viewpoints for exploration, while a global graph maintains the connections among the frontiers in the map. Cao \textit{et al.} \cite{tare, geo_exp_6} introduced hierarchical active 3D coverage and reconstruction which computes a coarse coverage path at the global scale followed by a local planner that ensures collision avoidance via a high-resolution path. As shown by Corah and Michael \cite{corah}, coverage-based exploration strategies can be formulated as mutual-information maximization policies in the absence of sensor noise. However, in many real-world circumstances sensor measurements are corrupted by non-trivial noise, reducing the effectiveness of geometric exploration methods that do not capture probabilistic uncertainty. For example, due to the domain shift between the training data and the test environment, utilizing a pre-trained semantic segmentation model in the mapping process requires accounting for measurement uncertainty in the exploration policy.

The work by Elfes \cite{info_exp_1} is among the first to propose an information-based utility function for measuring and minimizing map uncertainty. Information-based exploration strategies have been devised for uncertainty minimization in robot localization or environment mapping \cite{info_exp_2, info_exp_3, info_exp_4}. Information-theoretic objectives, however, require integration over the potential sensor measurements, limiting the use of direct numerical approximations to short planning horizons. Kollar and Roy~\cite{info_exp_5} formulated active mapping using an extended Kalman filter and proposed a local-global optimization, leading to significant gains in efficiency for uncertainty computation and long-horizon planning. Unlike geometry-based approaches, information-theoretic exploration can be directly formulated for active simultaneous localization and mapping (SLAM) \cite{active_slam_1,active_slam_2,active_slam_3,active_slam_4}, aiming to determine a sensing trajectory that balances robot state uncertainty and visitation of unexplored map regions. Stachniss \textit{et al.}~\cite{info_exp_7} approximate information gain for a Rao-blackwellized particle filter over the joint state of robot pose and map occupancy. Julian \textit{et al.}~\cite{julian} prove that, for range measurements and known robot position, the Shannon mutual information is maximized over trajectories that visit unexplored areas. However, without imposing further structure over the observation model, computing the mutual information objective requires numerical integration. The need for efficient mutual information computation becomes evident in 3-D environments. Cauchy-Schwarz quadratic mutual information (CSQMI)~\cite{csqmi} and fast Shannon mutual information (FSMI)~\cite{fsmi} offer efficiently computable closed-form objectives for active occupancy mapping with range measurements. Henderson \textit{et al.}~\cite{info_exp_11} propose an even faster computation based on a recursive expression for Shannon mutual information in continuous maps.

The recent success of machine learning methods of perception has motivated learning autonomous exploration policies. Chen \textit{et al.} \cite{rl_exp_1} attempt to bridge the sim2sim and sim2real gaps via graph neural networks and deep reinforcement learning. This enables decision-making over graphs containing relevant exploration information which is provided by human experts in order to predict a robot's optimal sensing action in belief space. Lodel \textit{et al.}~\cite{rl_exp_2} introduce a deep reinforcement learning policy which recommends next best view that maximizes information gain via defining mutual information as the training reward. Zwecher \textit{et al.}~\cite{rl_exp_3} employ deep reinforcement learning to find an exploration policy that plans collision-free coverage paths, while another neural network provides a predicted full map given the partially observed environment. Zhang \textit{et al.}~\cite{rl_exp_4} propose a multi-agent reinforcement learning exploration method where regions of interest, free space, and robots are represented as graph nodes, and hierarchical-hops graph neural networks (H2GNN) are used to identify key information in the environment. Related to multi-robot exploration, the authors in \cite{rl_exp_5} utilize an actor-critic strategy to map an unknown environment, where Voronoi partitioning divides the exploration regions among the robots. As this paper demonstrates, incorporating semantic uncertainty in addition to geometric information in the exploration process can be beneficial. Additionally, using Shannon mutual information as an objective function may help train more generalizable exploration policies because it mitigates the need for training sensor-specific models. Hence, the techniques proposed in this paper are complementary to learning approaches and can provide robustness to measurement uncertainty and domain shift caused by sensor and operational condition variations.

Active semantic mapping has recently attracted much attention due to the proliferation of fast object detection and semantic segmentation algorithms implemented on mobile robot platforms. The authors in \cite{sem_mapping_2} use a two-layer architecture, where the knowledge representation layer provides a belief over the environment state to the action layer, which subsequently chooses an action to gather information or execute a task. The work in \cite{sem_mapping_4} presents a semantic exploration policy which takes an occluded semantic point cloud of an object, finds a match in a database to estimate the full object dimensions, and then generates candidate next observation poses to reconstruct the object. The next best view is computed via a volumetric information gain metric that computes visible entropy from a candidate pose. The semantic map used in this paper is a collection of bounding boxes around objects. Active semantic mapping has also been employed to develop sample-efficient deep learning methods. Blum \textit{et al.}~\cite{sem_mapping_1} propose an active learning method for training semantic segmentation networks where the novelty (epistemic uncertainty) of the input images is estimated as the distance from the training data in the embedding space, while a path planning method maximizes novelty of future input images along the planned trajectory, assuming novel images are spatially correlated. Georgakis \textit{et al.}~\cite{sem_mapping_3} actively train a hierarchical semantic map generation model that predicts occupancy and semantics given occluded input. The authors use an ensemble of map generation models in order to predict epistemic uncertainty of the predicted map. The uncertainty is then used to choose trajectories for actively training the model with new images that differ the most with the training data of the current model. SSMI distinguishes itself from the aforementioned works by introducing a dense Bayesian multi-class mapping with a closed-form uncertainty measure, as opposed to sampling-based uncertainty estimation. Moreover, our information-theoretic objective function directly models sensor noise specifications, unlike volumetric information gain.

This paper is most related to CSQMI~\cite{csqmi} and FSMI~\cite{fsmi} in that it develops a closed-form expression for mutual information. However, instead of a binary map and range-only measurements, our formulation considers a multi-class map with Bayesian updates using range-category measurements. Since the same occupancy map can be derived from many different multi-class maps, the information associated with various object classes will fail to be captured if we solely rely on occupancy information, as the case in CSQMI and FSMI. Therefore, we expect to perform exploration more efficiently by using the multi-class perception model, and consequently, expanding the notion of uncertainty to multiple classes.

%% file: tex/problemStatement.tex
\section{Problem Statement}
\label{sec:problem_statement}

Consider a robot with pose $\bfX_t \in SE(3)$ at time $t$ and deterministic discrete-time kinematics:
\begin{equation}
    \bfX_t := \begin{bmatrix} \bfR_t & \bfp_t \\ \mathbf{0}^\top & 1 \end{bmatrix}, \qquad \bfX_{t+1} =  \bfX_{t}\exp{(\tau \hat{\bfu}_t)},
\label{eq:dynamic_model}
\end{equation}
where $\bfR_t \in SO(3)$ is the robot orientation, $\bfp_t \in \mathbb{R}^3$ is the robot position, $\tau$ is the time step, and $\bfu_t := [\bfv_t^\top, \bfomega_t^\top]^\top \in \calU \subset \mathbb{R}^6$ is the control input, consisting of linear velocity $\bfv_t \in \mathbb{R}^3$ and angular velocity $\bfomega_t \in \mathbb{R}^3$. The function $\hat{(\cdot)}: \bbR^6 \rightarrow \mathfrak{se}(3)$ maps vectors in $\mathbb{R}^6$ to the Lie algebra $\mathfrak{se}(3)$. See \cite[Chapter~7]{barfoot} for a definition of the Lie groups $SO(3)$ and $SE(3)$ and the corresponding Lie algebras $\mathfrak{so}(3)$ and $\mathfrak{se}(3)$. The robot is navigating in an environment consisting of a collection of disjoint sets $\calE_k \subset \mathbb{R}^3$, each associated with a semantic category $k \in \calK := \crl{0,1,\ldots,K}$. Let $\calE_0$ denote free space, while each $\calE_k$ for $k >0$ represents a different category, such as building, vegetation, terrain (see Fig.~\ref{fig:prob_state}).

We assume that the robot is equipped with a sensor that provides information about the distance to and semantic categories of surrounding objects along a set of rays $\crl{\bfeta_b}_b$, where $b$ is the ray index, $\bfeta_b \in \mathbb{R}^3$ with $\|\bfeta_b\|_2 = r_{max}$, and $r_{max} > 0$ is the maximum sensing range.

\begin{definition}
A \emph{sensor observation} at time $t$ from robot pose $\bfX_t$ is a collection $\calZ_t := \crl{\bfz_{t,b}}_b$ of range and category measurements $\bfz_{t,b}:= (r_{t,b}, y_{t,b}) \in \mathbb{R}_{\geq 0} \times \calK$, acquired along the sensor rays $\bfR_t\bfeta_b$ with $\bfeta_b \in \crl{\bfeta_b}_b$ at robot position $\bfp_t$. 
\end{definition}

Such information may be obtained by processing the observations of an RGBD camera or a Lidar with a semantic segmentation algorithm~\cite{bonnet}. Fig.~\ref{fig:prob_state} shows an example where each pixel in the RGB image corresponds to one sensor ray $\bfeta_b$, while its corresponding values in the semantic segmentation and the depth images encode category $y_{t,b}$ and range $r_{t,b}$, respectively. The goal is to construct a multi-class map $\bfm$ of the environment based on the labeled range measurements. We model $\bfm$ as a grid of cells $i \in \calI := \{1, \ldots, N\}$, each labeled with a category $m_i \in \calK$. In order to model noisy sensor observations, we consider a probability density function (PDF) $p(\calZ_t \mid \bfm, \bfX_t)$. This observation model allows integrating the measurements into a probabilistic map representation using Bayesian updates. Let $p_t(\bfm) := p(\bfm \mid \calZ_{1:t}, \bfX_{1:t})$ be the probability mass function (PMF) of the map $\bfm$ given the robot trajectory $\bfX_{1:t}$ and observations $\calZ_{1:t}$ up to time $t$. Given a new observation $\calZ_{t+1}$ obtained from robot pose $\bfX_{t+1}$, the Bayesian update to the map PMF is:
\begin{equation}
\label{eq:bayes_rule}
p_{t+1}(\bfm) \propto p(\calZ_{t+1} | \bfm, \bfX_{t+1}) p_t(\bfm).
\end{equation}
We assume that the robot pose is known and omit the dependence of the map distribution and the observation model on it for brevity. We consider the following problem.

\begin{problem}
Given a prior map PMF $p_t(\bfm)$ at time $t$ and a finite planning horizon $T$, maximize the ratio:
\begin{equation}
\label{eq:sem_exp}
\begin{aligned}
\max_{\bfu_{t:t+T-1}} \frac{I\prl{\bfm; \calZ_{t+1:t+T} \mid \calZ_{1:t}}}{J(\bfX_{t:t+T-1}, \bfu_{t:t+T-1})} 
 \;\;\text{subject to}\;\;  \eqref{eq:dynamic_model}, \eqref{eq:bayes_rule},
\end{aligned}
\end{equation}
of the mutual information $I\prl{\bfm; \calZ_{t+1:t+T} \mid \calZ_{1:t}}$ between the map $\bfm$ and future sensor observations $\calZ_{t+1:t+T}$ to the motion cost $J(\bfX_{t:t+T-1},\bfu_{t:t+T-1})$ of the control sequence $\bfu_{t:t+T-1}$.
\end{problem}

The definitions of the mutual information and motion cost terms in \eqref{eq:sem_exp} are:
\begin{align}
&I\prl{\bfm; \calZ_{t+1:t+T} | \calZ_{1:t}} := \!\sum_{\bfm \in \calK^N} \!\int \cdots \int p(\bfm, \calZ_{t+1:t+T} | \calZ_{1:t}) \notag\\
& \quad \times \log \frac{p(\bfm, \calZ_{t+1:t+T} \mid \calZ_{1:t})}{p(\bfm | \calZ_{1:t})p(\calZ_{t+1:t+T}|\calZ_{1:t})} \prod_{\tau = 1}^T \prod_b d\bfz_{t+\tau,b}, \label{eq:mutual-information}\\
&J(\bfX_{t:t+T-1}, \bfu_{t:t+T-1}) := q(\bfX_{t+T}) + \sum_{\tau = 0}^{T-1} c(\bfX_{t+\tau},\bfu_{t+\tau}), \notag
\end{align}
where the integration in \eqref{eq:mutual-information} is over all possible values of all sensor beams over all times $\bfz_{t+\tau,b}$, and the strictly positive terms $q(\bfX)$ and $c(\bfX,\bfu)$ model terminal and stage motion costs (e.g., distance traveled, elapsed time), respectively.

We develop a multi-class extension to the log-odds occupancy mapping algorithm \cite[Ch.~9]{ProbabilisticRoboticsBook} in Sec.~\ref{sec:bayes_multi_class_mapping} and derive an efficient approximation to the mutual information term in Sec.~\ref{sec:info_plan}. In Sec.~\ref{sec:info_comp_octomap}, we present a multi-class extension of the OctoMap \cite{octomap} algorithm, alongside a fast computation of mutual information over a semantic OctoMap using run-length encoding. This allows autonomous exploration of large 3-D environments by rapidly evaluating potential robot trajectories online and (re-)selecting the one that maximizes the objective in \eqref{eq:sem_exp}. Finally, in Sec.~\ref{sec:experiments}, we demonstrate the performance of our approach in simulated and real-world experiments.

%% file: tex/perception.tex
\section{Bayesian Multi-class Mapping}
\label{sec:bayes_multi_class_mapping}

\begin{figure*}[t]
    \captionsetup[subfigure]{justification=centering}
    \centering
    \begin{subfigure}[t]{0.49\linewidth}
        \includegraphics[width=\linewidth]{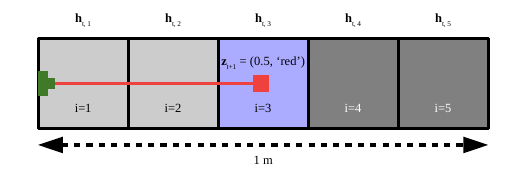}
        \caption{Map estimate at time $t$ with observation $\bfz_{t+1}$}
    \end{subfigure}%
    \hfill%
    \begin{subfigure}[t]{0.49\linewidth}
        \includegraphics[width=\linewidth]{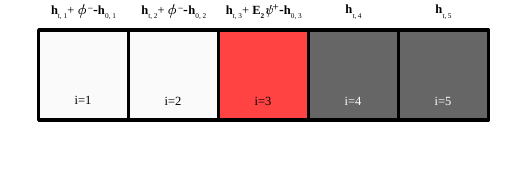}
        \caption{Posterior map estimate at time $t+1$}
    \end{subfigure}\\
    \begin{subfigure}[t]{\linewidth}
    \begin{subfigure}[t]{0.49\linewidth}
        \includegraphics[width=\linewidth]{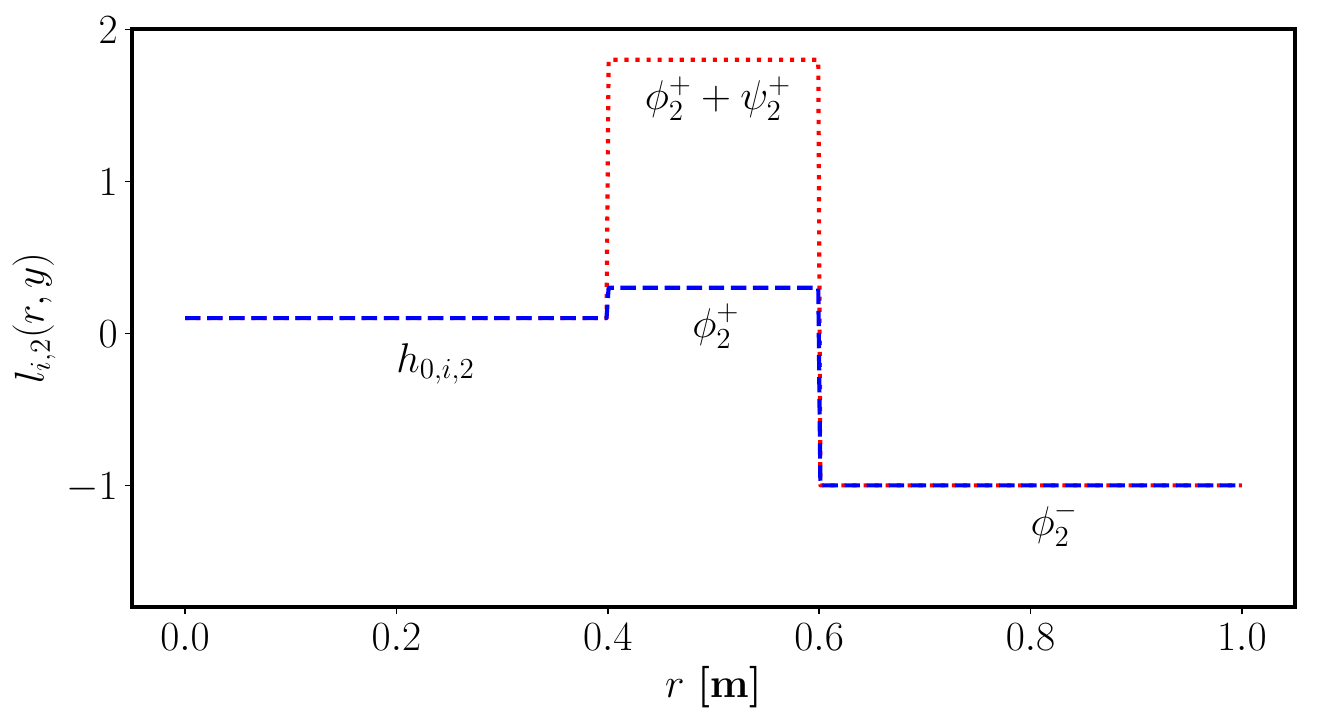}
    \end{subfigure}%
    \hfill%
    \begin{subfigure}[t]{0.49\linewidth}
        \includegraphics[width=\linewidth]{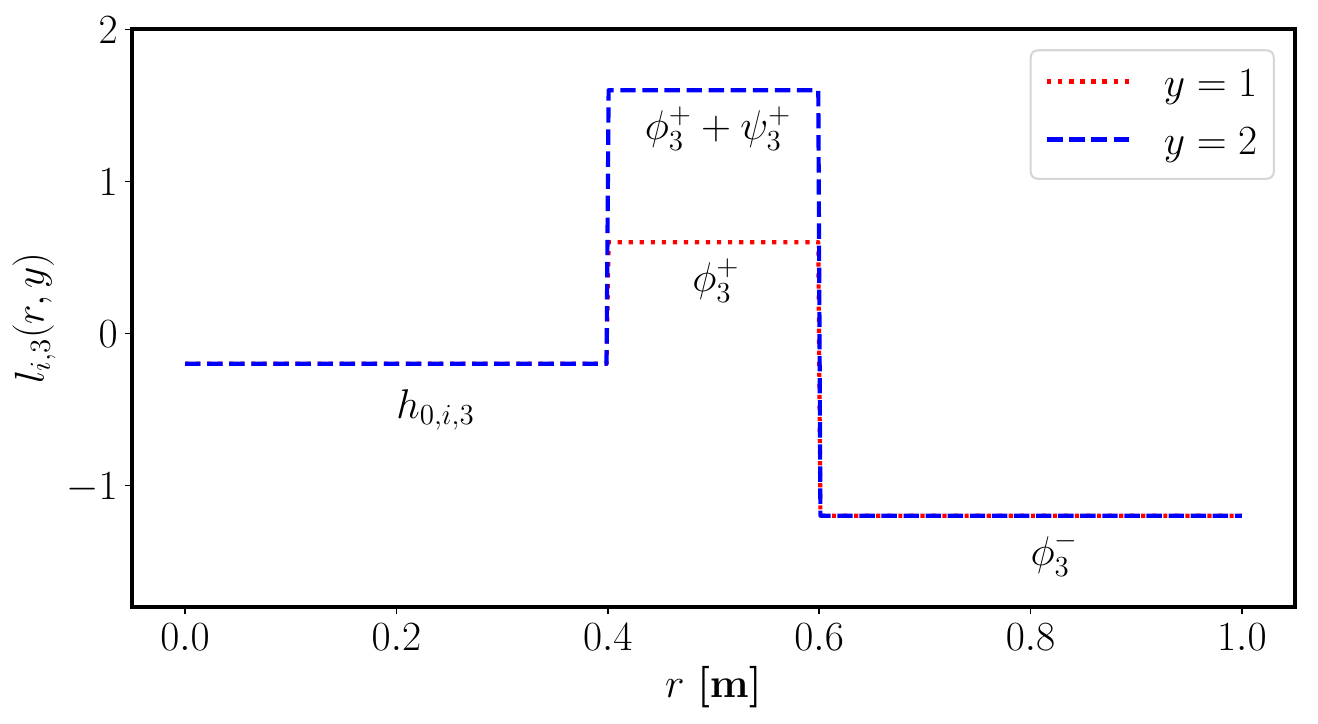}
    \end{subfigure}
    \caption{Inverse observation log-odds vector for $i = 3$}
    \end{subfigure}
    \vspace{0.01\textwidth}
    \caption{Illustration of the Bayesian multi-class mapping given a range-category observation $\bfz_{t+1} = (r_{t+1}, y_{t+1})$ for an environment with object classes of \textit{white} (free class), \textit{red}, and \textit{blue}, encoded as $y = 0$, $y = 1$, and $y = 2$, respectively: (a) Portion of the map $\bfm$ along the observation $\bfz_{t+1}$. Each cell has a multi-class log-odds vector $\bfh_{t, i} \in \bbR^3$ for $i \in \crl{1, \ldots, 5}$ at time $t$. The cell brightness encodes the occupancy probability, while the cell color represents the most likely category; (b) Map estimate at time $t + 1$ after update with $\bfz_{t+1}$. Note how each multi-class log-odds vector changes based on the inverse observation model $\bfl_i(\bfz_{t+1})$; (c) Second and third elements of the inverse observation log-odds vector $\bfl_i(\bfz_{t+1})$ for $i = 3$ as a function of range $r$ and category $y$ in observation $\bfz_{t+1}$. Note that the first element of $\bfl_i(\bfz_{t+1})$ is always zero.}
    \label{fig:perception}
\end{figure*}

This section derives the Bayesian update in \eqref{eq:bayes_rule}, using a \emph{multinomial logit model} to represent the map PMF $p_t(\bfm)$ where each cell $m_i$ of the map stores the probability of object classes in $\calK$. To ensure that the number of parameters in the model scales linearly with the map size $N$, we maintain a factorized PMF over the cells:
\begin{equation}
\label{eq:pdf_factorization}
p_t(\bfm) = \prod_{i=1}^N p_t(m_i).
\end{equation}
We represent the individual cell PMFs $p_t(m_i)$ over the semantic categories $\calK$ using a vector of log odds:
\begin{equation}
\bfh_{t,i} := \begin{bmatrix} \log \frac{p_t(m_i = 0)}{p_t(m_i = 0)} & \cdots & \log \frac{p_t(m_i = K)}{p_t(m_i = 0)} \end{bmatrix}^\top \!\!\in \mathbb{R}^{K+1},
\end{equation}
where the free-class likelihood $p_t(m_i = 0)$ is used as a pivot. Given the log-odds vector $\bfh_{t,i}$, the PMF of cell $m_i$ may be recovered using the softmax function $\sigma:\mathbb{R}^{K+1} \mapsto \mathbb{R}^{K+1}$:
\begin{equation}
p_t(m_i = k) = \sigma_{k+1}(\bfh_{t,i}) := \frac{\bfe_{k+1}^\top \exp(\bfh_{t,i})}{ \mathbf{1}^\top \exp(\bfh_{t,i})},
\end{equation}
where $\bfe_k$ is the standard basis vector with $k$-th element equal to $1$ and $0$ elsewhere, $\mathbf{1}$ is the vector with all elements equal to $1$, and $\exp(\cdot)$ is applied elementwise to the vector $\bfh_{t,i}$. To derive Bayes rule for the log-odds $\bfh_{t,i}$, we need to specify an observation model for the range and categry measurements.

\begin{definition}
The \emph{inverse observation model} of a range-category measurement $\bfz$ obtained from robot pose $\bfX$ along sensor ray $\bfeta$ with respect to map cell $m_i$ is a probability mass function $p(m_i | \bfz; \bfX, \bfeta)$.
\end{definition}

The Bayesian update in \eqref{eq:bayes_rule} for $\bfh_{t,i}$ can be obtained in terms of the range-category inverse observation model, evaluated at a new measurement set $\calZ_{t+1}$.

\begin{proposition}
\label{prop:log-odds-bayes-rule}
Let $\bfh_{t,i}$ be the log odds of cell $m_i$ at time $t$. Given sensor observation $\calZ_{t+1}$, the posterior log-odds are:
\begin{equation}
\label{eq:log-odds-bayes-rule}
\bfh_{t+1,i} = \bfh_{t,i} + \sum_{\bfz \in \calZ_{t+1}} \prl{ \bfl_i(\bfz) - \bfh_{0,i}}
\end{equation}
where $\bfl_i(\bfz)$ is the inverse observation model log odds:
\begin{equation}\label{eq:log-ods-vector}
\bfl_i(\bfz) := \begin{bmatrix} \log \frac{p(m_i = 0 | \bfz)}{p(m_i = 0 | \bfz)} & \cdots & \log \frac{p(m_i = K | \bfz)}{p(m_i = 0 | \bfz)} \end{bmatrix}^\top.
\end{equation}
\end{proposition}

\begin{proof}
See Appendix~\ref{app:log-odds-bayes-rule}.
\end{proof}

To complete the Bayesian multi-class mapping algorithm suggested by \eqref{eq:log-odds-bayes-rule} we need a particular inverse observation model. When a sensor measurement is generated, the sensor ray continues to travel until it hits an obstacle of category $\calK \setminus \{0\}$ or reaches the maximum sensing range $r_{max}$. The labeled range measurement $\bfz = (r,y)$ obtained from position $\bfp$ with orientation $\bfR$ indicates that map cell $m_i$ is occupied if the measurement end point $\bfp + \frac{r}{r_{max}} \bfR \bfeta$ lies in the cell. If $m_i$ lies along the sensor ray but does not contain the end point, it is observed as free. Finally, if $m_i$ is not intersected by the sensor ray, no information is provided about its occupancy. The map cells along the sensor ray can be determined by a rasterization algorithm, such as Bresenham's line algorithm \cite{bresenham}. We parameterize the inverse observation model log-odds vector in \eqref{eq:log-ods-vector} as:
\begin{gather}
\label{eq:log_inverse_observation_model}
\bfl_i((r,y)) := \begin{cases}
\bfphi^+ + \bfE_{y+1}\bfpsi^+, & \text{$r$ indicates $m_i$ is occupied},\\
\bfphi^-, & \text{$r$ indicates $m_i$ is free},\\
\bfh_{0,i}, & \text{otherwise},
\end{cases}
\raisetag{3ex}
\end{gather}
where $\bfE_k := \bfe_k \bfe_k^\top$ and $\bfpsi^+\!, \bfphi^-\!, \bfphi^+ \in \mathbb{R}^{K+1}$ are parameter vectors, whose first element is $0$ to ensure that $\bfl_i(\bfz)$ is a valid log-odds vector. This parameterization leads to an inverse observation model $p(m_i = k | \bfz) = \sigma_{k+1}(\bfl_i(\bfz))$, which is piecewise constant along the sensor ray. Fig.~\ref{fig:perception} illustrates our Bayesian multi-class mapping method.

To compute the mutual information between an observation sequence $\calZ_{t+1:t+T}$ and the map $\bfm$ in the next section, we will also need the PDF of a range-category measurement $\bfz_{\tau,b} \in \calZ_{t+1:t+T}$ conditioned on $\calZ_{1:t}$. Let $\calR_{\tau,b}(r) \subset \calI$ denote the set of map cell indices along the ray $\bfR_\tau \bfeta_b$ from robot position $\bfp_\tau$ with length $r$. Let $\gamma_{\tau,b}(i)$ denote the distance traveled by the ray $\bfR_\tau \bfeta_b$ within cell $m_i$ and $i_{\tau,b}^* \in \calR_{\tau,b}(r)$ denote the index of the cell hit by $\bfz_{\tau,b}$. We define the PDF of $\bfz_{\tau,b} = (r, y)$ conditioned on $\calZ_{1:t}$ as:
\begin{gather}
\label{eq:cond_prob_approx}
p(\bfz_{\tau,b} | \calZ_{1:t}) = \frac{p_t(m_{i_{\tau,b}^*} = y)}{\gamma_{\tau,b}(i_{\tau,b}^*)} \negquad \prod_{i \in \calR_{\tau,b}(r) \setminus \{i_{\tau,b}^*\}} \negquad p_t(m_i = 0).
\raisetag{2ex}
\end{gather}
This definition states that the likelihood of $\bfz_{\tau,b}= (r,y)$ at time $t$ depends on the likelihood that the cells $m_i$ along the ray $\bfR_{\tau}\bfeta_b$ of length $r$ are empty and the likelihood that the hit cell $m_{i_{\tau,b}^*}$ has class $y$. A similar model for binary observations has been used in \cite{julian, csqmi, fsmi}.

This section described how an observation affects the map PMF $p_t(\bfm)$. Now, we switch our focus to computing the mutual information between a sequence of observations $\calZ_{t+1:t+T}$ and the multi-class occupancy map $\bfm$.

%% file: tex/informativePlanning.tex
\section{Informative Planning}
\label{sec:info_plan}

Proposition~\ref{prop:log-odds-bayes-rule} allows a multi-class formulation of occupancy grid mapping, where the uncertainty of a map cell depends on the probability of each class $p_t(m_i=k)$, instead of only the binary occupancy probability $1 - p_t(m_i = 0)$. Moreover, the inverse observation model in \eqref{eq:log_inverse_observation_model} may contain different likelihoods for the different classes which can be used to prioritize the information gathering for specific classes. Fig.~\ref{fig:mut_info_comp_a} shows an example where the estimated map of an environment with $3$ classes, $\textit{free}$, $\textit{class}_1$, $\textit{class}_2$, contains two regions with similar occupancy probability but different semantic uncertainty. In particular, the red and green walls have the same occupancy probability of $0.9$, as shown in Fig.~\ref{fig:mut_info_comp_b}, but the red region more certainly belongs to $\textit{class}_1$ and the green region has high uncertainty between the two classes. As can be seen in Fig.~\ref{fig:mut_info_comp_c}, the mutual information associated with a binary occupancy map cannot distinguish between the red and green regions since they both have the same occupancy probability. In contrast, the multi-class map takes into account the semantic uncertainty among different categories, as can be seen in Fig.~\ref{fig:mut_info_comp_d} where the uncertain green region has larger mutual information than the certain red region.

\begin{figure}[t]
    \captionsetup[subfigure]{justification=centering}
    \begin{subfigure}[t]{0.47\linewidth}
        \includegraphics[width=\linewidth]{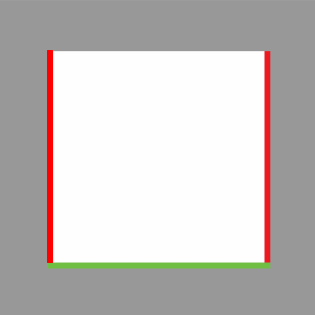}%
        \caption{}
        \label{fig:mut_info_comp_a}
    \end{subfigure}%
    \hfill%
    \begin{subfigure}[t]{0.47\linewidth}
        \centering
        \includegraphics[width=\linewidth]{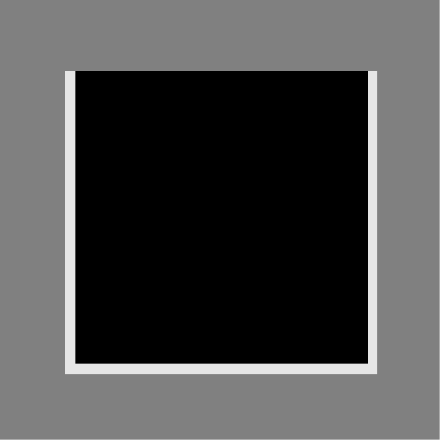}%
        \caption{}
        \label{fig:mut_info_comp_b}
    \end{subfigure}\\
    \begin{subfigure}[t]{0.47\linewidth}
        \includegraphics[width=\linewidth]{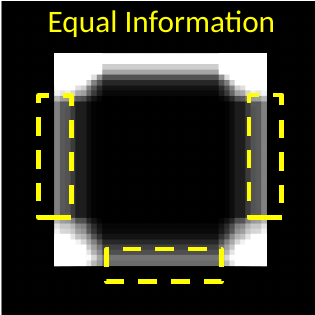}%
        \caption{}
        \label{fig:mut_info_comp_c}
    \end{subfigure}%
    \hfill%
    \begin{subfigure}[t]{0.47\linewidth}
        \includegraphics[width=\linewidth]{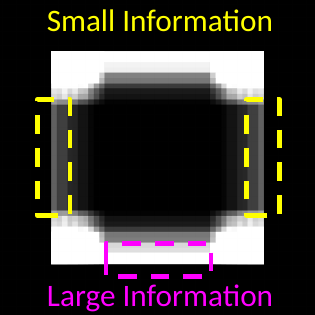}%
        \caption{}
        \label{fig:mut_info_comp_d}
    \end{subfigure}
  \caption{Comparison between the information surfaces of binary and multi-class map representations. (a) Environment with three classes $\textit{free}$, $\textit{class}_1$, and $\textit{class}_2$ where the white and gray regions represent free and unknown space, respectively, with $\bfp^{white}(m_i) = \left[1, 0, 0\right]$ and $\bfp^{gray}(m_i) = \left[0.3, 0.3, 0.3\right]$. The red and green regions have the same occupancy probability of $p(m_i = \text{occupied}) = 0.9$ but different class uncertainty, i.e., $\bfp^{red}(m_i) = \left[0.1, 0.8, 0.1\right]$ and $\bfp^{green}(m_i) = \left[0.1, 0.45, 0.45\right]$. (b) Binary occupancy map, where the intensity of each pixel is proportional to its occupancy probability, regardless of object class. (c) Occupancy mutual information surface. (d) Semantic mutual information surface. Each pixel in the information surfaces shows the value of mutual information between the map and a set of range-category observations, uniformly sampled from a $360^{\circ}$ field of view at each pixel location.}
  \label{fig:mut_info_comp}
\end{figure}

These observations suggest that more accurate uncertainty quantification may be achieved using a multi-class instead of a binary perception model, potentially enabling a more efficient exploration strategy. However, computing the mutual information term in \eqref{eq:mutual-information} is challenging because it involves integration over all possible values of the observation sequence $\calZ_{t+1:t+T}$. Our main result is an efficiently-computable lower bound on $I\prl{\bfm; \calZ_{t+1:t+T} | \calZ_{1:t}}$ for range-category observations $\calZ_{t+1:t+T}$ and a multi-class occupancy map $\bfm$. The result is obtained by selecting a subset $\underline{\calZ}_{t+1:t+T} = \crl{\bfz_{\tau,b}}_{\tau=t+1,b=1}^{t+T,B}$ of the observations $\calZ_{t+1:t+T}$ in which the sensor rays are non-overlapping. Precisely, any pair of measurements $\bfz_{\tau,b}$, $\bfz_{\tau',b'} \in \underline{\calZ}_{t+1:t+T}$ satisfies:
\begin{equation}
\label{eq:nonoverlapping}
\calR_{\tau,b}(r_{max}) \cap \calR_{\tau',b'}(r_{max}) = \emptyset.
\end{equation}
In practice, constructing $\underline{\calZ}_{t+1:t+T}$ requires removing intersecting rays from $\calZ_{t+1:t+T}$ to ensure that the remaining observations are mutually independent. The mutual information between $\bfm$ and $\underline{\calZ}_{t+1:t+T}$ can be obtained as a sum of mutual information terms between single rays $\bfz_{\tau,b} \in \underline{\calZ}_{t+1:t+T}$ and map cells $m_i$ observed by $\bfz_{\tau,b}$. This idea is inspired by CSQMI~\cite{csqmi} but we generalize it to multi-class observations and map.

\begin{proposition}
\label{prop:mut_inf_semantic}
Given a sequence of labeled range observations $\calZ_{t+1:t+T}$, let $\underline{\calZ}_{t+1:t+T} = \crl{\bfz_{\tau,b}}_{\tau=t+1,b=1}^{t+T,B}$ be a subset of non-overlapping measurements that satisfy \eqref{eq:nonoverlapping}. Then, the Shannon mutual information between $\calZ_{t+1:t+T}$ and a multi-class occupancy map $\bfm$ can be lower bounded as:
\begin{equation}
\label{eq:mut_inf_semantic}
\begin{aligned}
I(\bfm; &\calZ_{t+1:t+T} | \calZ_{1:t}) \geq I\prl{\bfm; \underline{\calZ}_{t+1:t+T} | \calZ_{1:t}}\\
&=\sum_{\tau=t+1}^{t+T} \sum_{b=1}^{B} \sum_{k=1}^K \sum_{n=1}^{N_{\tau,b}} p_{\tau,b}(n, k) C_{\tau,b}(n, k),
\end{aligned}
\end{equation}
where $N_{\tau,b} := | \calR_{\tau,b}(r_{max}) |$,
\begin{equation*}
\begin{aligned}
p_{\tau,b}(n, k) &:= p_t(m_{i_{\tau,b}^*} = k) \prod_{i \in \Tilde{\calR}_{\tau,b}(n) \setminus \{i_{\tau,b}^*\}} p_t(m_i = 0),\\
C_{\tau,b}(n, k) &:= f(\bfphi^+ + \bfE_{k+1}\bfpsi^+ - \bfh_{0,i_{\tau,b}^*}, \bfh_{t,i_{\tau,b}^*})\\
& + \sum_{i \in \Tilde{\calR}_{\tau,b}(n) \setminus \{i_{\tau,b}^*\}} f(\bfphi^--\bfh_{0,i}, \bfh_{t,i}),\\
f(\bfphi, \bfh) &:= \log\prl{ \frac{\mathbf{1}^\top \exp(\bfh)}{\mathbf{1}^\top \exp(\bfphi + \bfh)} } + \bfphi^\top \sigma(\bfphi + \bfh),
\end{aligned}
\end{equation*}
and $\Tilde{\calR}_{\tau,b}(n) \subseteq \calR_{\tau,b}(r_{max})$ is the set of the first $n$ map cell indices along the ray $\bfR_{\tau}\bfeta_b$, i.e., $\Tilde{\calR}_{\tau,b}(n) := \{i \mid i \in \calR_{\tau,b}(r), |\calR_{\tau,b}(r)| = n, r \leq r_{max}\}$.
\end{proposition}

\begin{proof}
See Appendix~\ref{app:mut-inf-semantic}.
\end{proof}

In \eqref{eq:mut_inf_semantic}, $p_{\tau, b}(n, k)$ represents the probability that the $n$-th map cell along the ray $\bfR_\tau \bfeta_b$ belongs to object category $k$ while all of the previous cells are free. The function $f(\bfphi, \bfh)$ denotes the log-ratio of the map PMF $\sigma(\bfh)$ and its posterior $\sigma(\bfphi + \bfh)$, averaged over object categories in $\calK$ (see \eqref{eq:f_func} in Appendix~\ref{app:mut-inf-semantic} for more details). As a result, $C_{\tau, b}(n, k)$ is the sum of log-ratios for the first $n$ cells along the ray $\bfR_\tau \bfeta_b$ under the same event as the one $p_{\tau, b}(n, k)$ is associated with. Therefore, the lower bound $I\prl{\bfm; \underline{\calZ}_{t+1:t+T} | \calZ_{1:t}}$ is equivalent to the expectation of summed log-ratios $C_{\tau, b}(n, k)$ over all possible instantiations of the observations in $\underline{\calZ}_{t+1:t+T}$.

Proposition~\ref{prop:mut_inf_semantic} allows evaluating the informativeness according to \eqref{eq:sem_exp} of any potential robot trajectory $\bfX_{t:t+T}$, $\bfu_{t:t+T-1}$. In order to perform informative planning, first, we identify the boundary between the explored and unexplored regions of the map, similar to \cite{frontier}. This can be done efficiently using edge detection, for example. Then, we cluster the corresponding map cells by detecting the connected components of the boundary. Each cluster is called a \textit{frontier}. A motion planning algorithm is used to obtain a set of pose trajectories to the map frontiers, determined from the current map PMF $p_t(\bfm)$. Alg.~\ref{alg:sem_exp} summarizes the procedure for determining a trajectory $\bfX^*_{t:t+T}$, $\bfu^*_{t:t+T-1}$ that maximizes the objective in \eqref{eq:sem_exp}, where $J(\bfX_{t:t+T-1}, \bfu_{t:t+T-1})$ is the length of the corresponding path. This kinematically feasible trajectory can be tracked by a low-level controller that takes the robot dynamics into account.

Evaluation of the mutual information lower bound in Proposition~\ref{prop:mut_inf_semantic} can be accelerated without loss in accuracy for map cells along the observation rays that contain equal PMFs. In the next section, we investigate this property of the proposed lower bound within the context of OcTree-based representations. We begin with proposing a multi-class version of the OctoMap technique, where map cells with equal multi-class probabilities can be compressed into a larger voxel. Next, a fast semantic mutual information formula is presented based on compression of range-category ray-casts over OcTree representations.

\begin{algorithm}[t]
\caption{Information-theoretic Path Planning}\label{alg:sem_exp}
\begin{algorithmic}[1]
  \renewcommand{\algorithmicrequire}{\textbf{Input:}}
  \renewcommand{\algorithmicensure}{\textbf{Output:}}
  \Require robot pose $\bfX_t$, map estimate $p_t(\bfm)$
  \State $\calF = \Call{findFrontiers}{p_t(\bfm)}$
  \For{$f \in \calF$}
    \State $\bfX_{t+1:t+T}, \bfu_{t:t+T-1} = \Call{planPath}{\bfX_t,p_t(\bfm),f}$
    \State Compute \eqref{eq:sem_exp} over $\bfX_{t:t+T}, \bfu_{t:t+T-1}$ via \eqref{eq:mut_inf_semantic}
  \EndFor
  \State \Return $\bfX^*_{t:t+T}$, $\bfu^*_{t:t+T-1}$ with the highest value
\end{algorithmic}
\end{algorithm}

%% file: tex/octomap.tex
\section{Information Computation for Semantic OctoMap Representations}
\label{sec:info_comp_octomap}

\begin{figure*}[t]
    \captionsetup[subfigure]{justification=centering}
    \centering
    \begin{subfigure}[t]{0.36\linewidth}
    \centering
    \includegraphics[width=\linewidth]{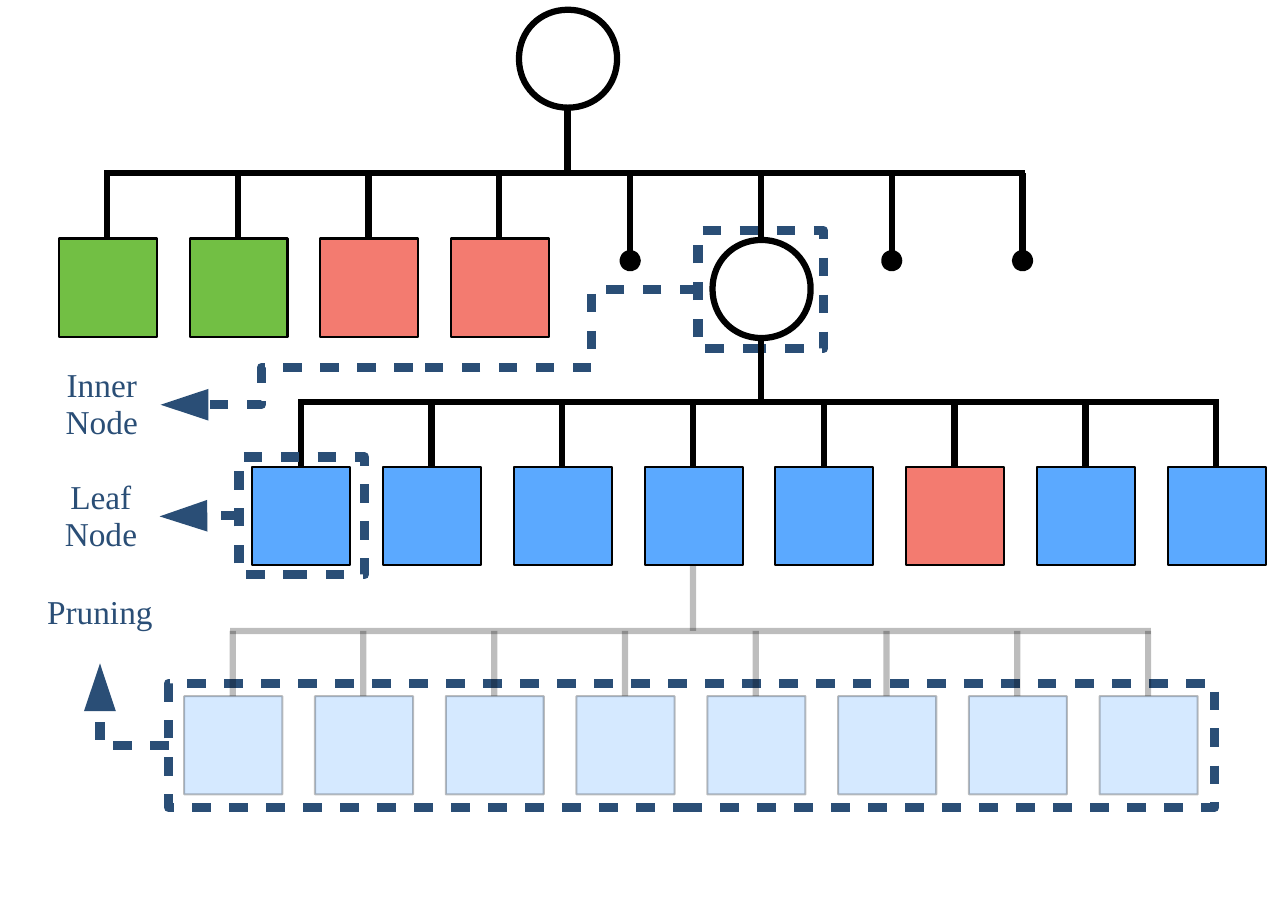}
    \caption{}
    \label{fig:OcTree_a}
    \end{subfigure}%
    \hfill%
    \begin{subfigure}[t]{0.60\linewidth}
    \centering
    \includegraphics[width=\linewidth]{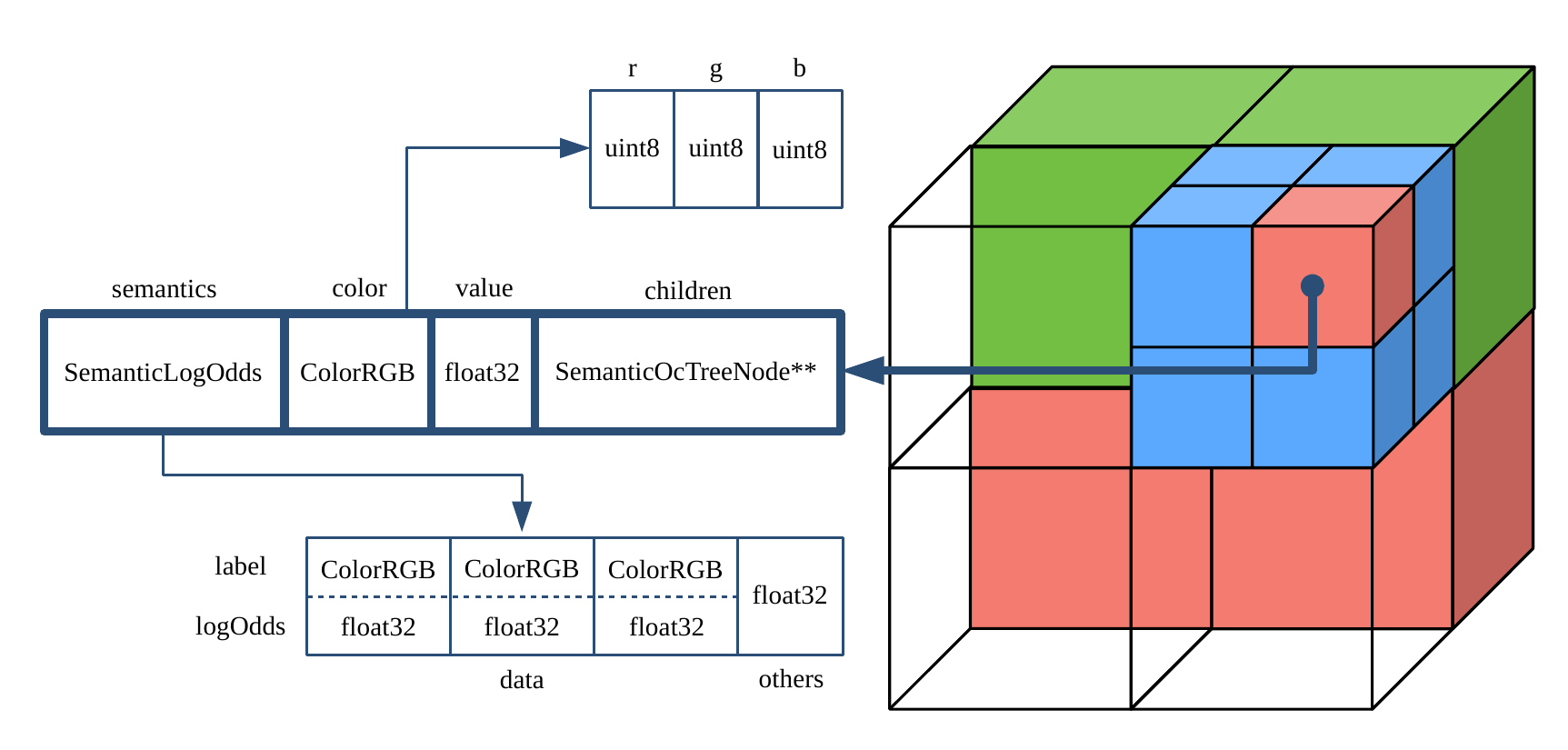}
    \caption{}
    \label{fig:OcTree_b}
    \end{subfigure}
    \caption{Example of a semantic OctoMap: (a) A white circle represents an \textit{inner} node such that its children collectively cover the same physical space as the inner node itself. A colored square represents a partition of the 3-D space where all downstream nodes contain identical semantic and occupancy values; therefore, they can be \textit{pruned} into a \textit{leaf} node. Lastly, black dots represent \textit{unexplored} spaces of the environment. (b) geometric representation of the same OcTree with an overview of the \textit{SemanticOcTreeNode} class.}
    \label{fig:OcTree}
\end{figure*}

Utilizing a regular-grid discretization to represent a 3-D environment has prohibitive storage and computation requirements. Large continuous portions of many real environments are unoccupied, suggesting that adaptive discretization is significantly more efficient. OctoMap \cite{octomap} is a probabilistic 3-D mapping technique that utilizes an OcTree data structure to obtain adaptive resolution, e.g., combining many small cells associated with free space into few large cells. In this section, we develop a multi-class version of OctoMap and propose an efficient multi-class mutual information computation which benefits from the OcTree structure.

\subsection{Semantic OctoMap}

An OcTree is a hierarchical data structure containing nodes that represent a section of the physical environment. Each node has either 0 or 8 children, where the latter corresponds to the 8 octants of the Euclidean 3-D coordinate system. Thus, the children of a parent node form an eight-way octant partition of the space associated with the parent node. Fig.~\ref{fig:OcTree} shows an example of a multi-class OcTree data structure.

We implement a \textit{SemanticOcTreeNode} class as a building block of the multi-class OcTree structure. A \textit{SemanticOcTreeNode} instance stores occupancy, color, and semantic information of its corresponding physical space, as shown in Fig.~\ref{fig:OcTree_b}. The most important data members of the \textit{SemanticOcTreeNode} class are:
\begin{itemize}
  \setlength\itemsep{0.5ex}
  \item \textit{children}: an array of pointers to \textit{SemanticOcTreeNode} storing the memory addresses of the $8$ child nodes,
  \item \textit{value}: a \textit{float} variable storing the log-odds occupancy probability of the node,
  \item \textit{color}: a \textit{ColorRGB} object storing the RGB color of the node,
  \item \textit{semantics}: a \textit{SemanticLogOdds} object maintaining a categorical probability distribution over the semantic labels in the form of a log-odds ratio.
\end{itemize}
For performance reasons, the \textit{SemanticLogOdds} class only stores the multi-class log-odds for the $3$ most likely class labels, with each label represented by a unique RGB color. In this case, the log-odds associated with the rest of the labels lump into a single \textit{others} variable. This relives the multi-class OcTree implementation from dependence on the number of labels that the object classifier can detect. Moreover, it significantly improves the speed of the mapping algorithm in cases with many semantic categories. See Sec.~\ref{subsec:octomap_class_comp} for an analysis of mapping time versus the number of stored classes.

The implementation of the multi-class OcTree is completed by defining a \textit{SemanticOcTree} class, which is derived from the \textit{OccupancyOcTreeBase} class of the OctoMap library \cite{octomap} and uses a \textit{SemanticOcTreeNode} as its node type. Fig.~\ref{fig:UML} illustrates the derivation of the \textit{SemanticOcTree} and \textit{SemanticOcTreeNode} classes as a UML diagram.

\begin{figure}[t]
  \centering
  \includegraphics[width=\linewidth]{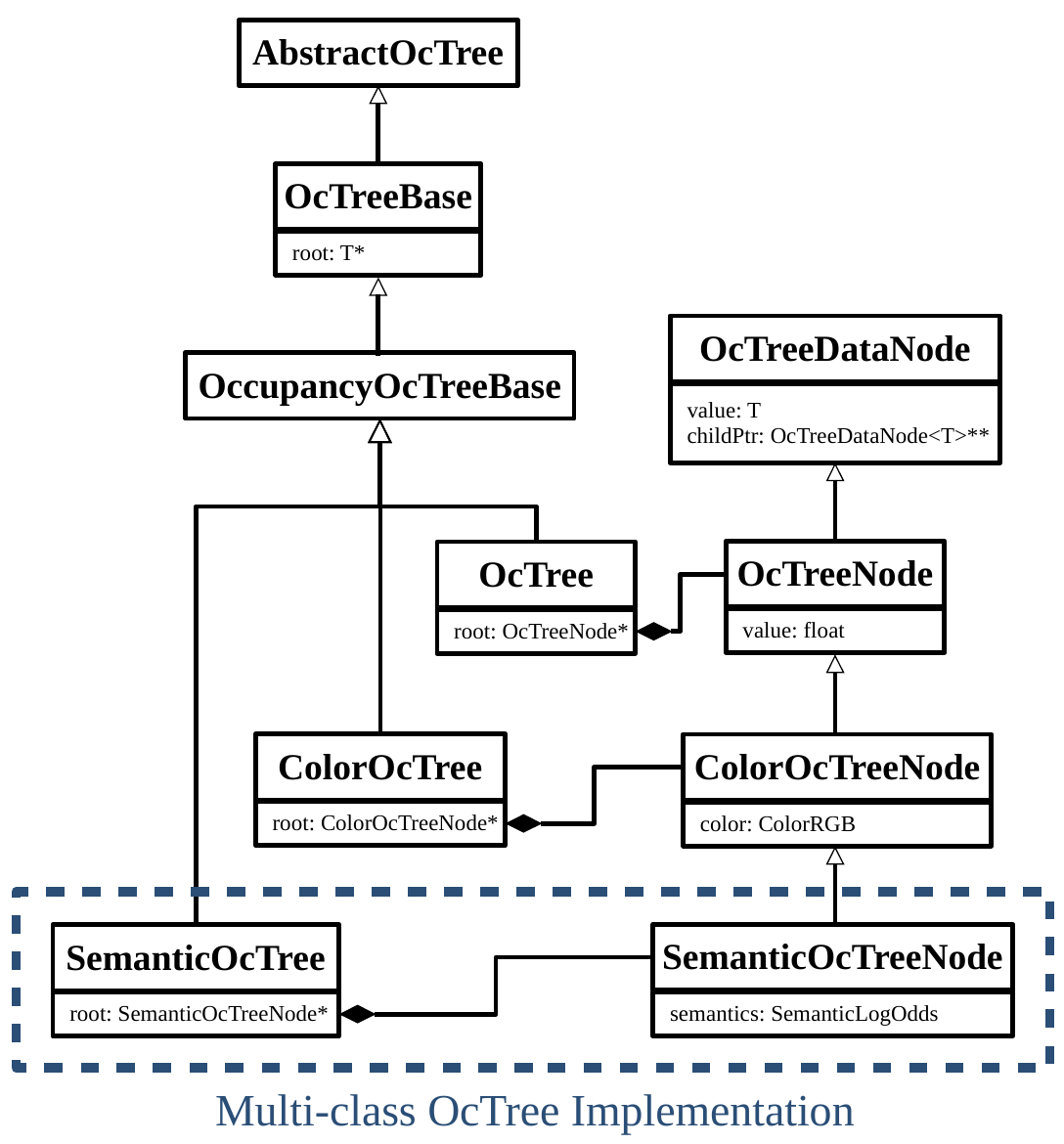}
  \caption{UML diagram showing the class inheritance used for the implementation of a multi-class OcTree.}
\label{fig:UML}
\end{figure}

In order to register a new observation to a multi-class OcTree, we follow the standard ray-casting procedure over an OcTree, as in \cite{octomap}, to find the observed leaf nodes. Then, for each observed leaf node, if the observation demands an update, the leaf node is recursively expanded to the smallest resolution and the multi-class log-odds of the downstream nodes are updated using \eqref{eq:log-odds-bayes-rule}. At the ray's end point, which indicates an occupied cell, we also update the \textit{color} variable by averaging the observed color with the current stored color of the corresponding node. Alg.~\ref{alg:fuse_obs} details the Bayesian update procedure for the multi-class OcTree.

\begin{algorithm}[t]
\caption{Multi-class OcTree Update of Node $\bfn_i$}\label{alg:fuse_obs}
\begin{algorithmic}[1]
 \renewcommand{\algorithmicrequire}{\textbf{Input:}}
 \renewcommand{\algorithmicensure}{\textbf{Output:}}
 \Require OcTree node $\bfn_i$, observation $\bfz = (r,y)$, mixing coefficient $\alpha$
 \State $s = \bfn_i.semantics$
 \State $s.d = \bfn_i.semantics.data$
 \State $s.o = \bfn_i.semantics.others$
 \If{$\bfz$ indicates free}
    \State Update $s$ with $\phi^-$
 \ElsIf{$\bfz$ indicates class $y$}
    \If{class $y$ is among the 3 most likely classes in $s$}
        \State Update $s$ with $\bfphi^+ + \bfE_{y+1}\bfpsi^+$
    \Else
        \LineComment{Derive $h_{aux}$ as a portion $\alpha$ of \textit{others} class}
        \State $h_{aux} = s.o + \log{\alpha}$ 
        \State $s.o \mathrel{+}= \phi^+_{others} + \log{(1-\alpha)}$
        \State $s_c = \Call{concat}{s.d, (y, h_{aux})}$
        \State Update $s_c$ with $\bfphi^+ + \bfE_{y+1}\bfpsi^+$
        \State Perform descending sort on $s_c$ with respect to log-odds values\label{alg:sorting}
        \LineComment{Pick 3 most likely classes}
        \State $s.d = s_c[0:2]$
        \LineComment{Combine the least likely class with \textit{others} class}
        \State $s.o = \log\Big(\exp{(s_c[3])} + \exp{(s.o)}\Big)$ \label{alg:others}
    \EndIf
 \EndIf
 \LineComment{Apply thresholds $\underline{s}$ and $\overline{s}$ for log-odds values}
 \State $s_f \leftarrow \min\crl{\max\crl{s_f, \underline{s}},\overline{s}}$
 \State $\bfn_i.semantics = s$
 \State \Return $\bfn_i$
\end{algorithmic}
\end{algorithm}

\begin{algorithm}[t]
\caption{Semantic Fusion of Two Child Nodes}\label{alg:semantic_fusion}
\begin{algorithmic}[1]
 \renewcommand{\algorithmicrequire}{\textbf{Input:}}
 \renewcommand{\algorithmicensure}{\textbf{Output:}}
 \Require OcTree nodes $\bfn_i$ and $\bfn_j$
 \State $s_i = \bfn_i.semantics$
 \State $s_j = \bfn_j.semantics$
 \LineComment{Non-repeating list of classes in $s_i$ and $s_j$}
 \State $\mathcal{K}_f = \Call{uniqueClass}{s_i, s_j}$
 \LineComment{Object instantiation for the fused semantics}
 \State $s_{f} = SemanticLogOdds()$
 \LineComment{Slice $s_i.o$ into smaller probabilities}
 \State $o_i = s_i.o - \log(1 + \mathcal{K}_f.size - s_i.d.size)$
 \State $o_j = s_j.o - \log(1 + \mathcal{K}_f.size - s_j.d.size)$
 \For{$y \in \mathcal{K}_f$}
    \If{$y \notin s_i.d.label \wedge y \in s_j.d.label$}
        \State $s_f.d.\Call{append}{y, \frac{o_i + s_j.d[y].logOdds}{2}}$
    \ElsIf{$y \in s_i.d.label \wedge y \notin s_j.d.label$}
        \State $s_f.d.\Call{append}{y, \frac{s_i.d[y].logOdds + o_j}{2}}$
    \Else
        \State $s_f.d.\Call{append}{y, \frac{s_i.d[y].logOdds + s_j.d[y].logOdds}{2}}$
    \EndIf
 \EndFor
 \State Perform descending sort on $s_f.d$ with respect to log-odds values
 \State $expOthers = \exp(\frac{o_i + o_j}{2})$
 \For{$i > 3$}
    \State $expOthers \mathrel{+}= \exp(s_f.d[i].logOdds)$
 \EndFor
 \State $s_f.d[3:end].\Call{remove}{ }$
 \State $s_f.o = \log(expOthers)$
 \State $s_f \leftarrow \min\crl{\max\crl{s_f, \underline{s}},\overline{s}}$
 \State \Return $s_f$
\end{algorithmic}
\end{algorithm}

To obtain a compressed OctoMap, it is necessary to define a rule for information fusion from child nodes towards parent nodes.
Depending on the application, different information fusion strategies may be implemented. For example, a conservative strategy would assign the multi-class log-odds of the child node with the highest occupancy probability to the parent node. In this work, we simply assign the average log-odds vector of the child nodes to their parent node as shown in Alg.~\ref{alg:semantic_fusion}. The benefit of an OctoMap representation is the ability to combine similar cells (leaf nodes) into a large cell (inner node).
This is called \textit{pruning} the OcTree. Every time after an observation is integrated to the map, starting from the deepest inner node, we check for each inner node if \begin{inparaenum}[1)] \item the node has 8 children, \item its children do not have any children of their own, and \item its children all have equal multi-class log-odds\end{inparaenum}. If an inner node satisfies all of these three criteria, its children are pruned and the inner node is converted into a leaf node with the same multi-class log-odds as its children. This helps to compress the majority of the free cells into a few large cells, while the occupied cells usually do not undergo pruning since only their surfaces are observed by the sensor and their inside remains an unexplored region. Due to sensor noise, it is unlikely that cells belonging to the same class (e.g., free or occupied by the same obstacle) attain identical multi-class log-odds. Maximum and minimum limits for the elements of the multi-class log-odds are used so that each cell arrives at a stable state as its multi-class log-odds entries reach the limits. Stable cells are more likely to share the same multi-class probability distribution, consequently increasing the chance of OcTree pruning. However, thresholding causes loss of information near $p_t(m_i=k) = 1$, $k \in \calK$ which can be controlled by the maximum and minimum limits.

\subsection{Information Computation}

A ray cast through an OcTree representation may visit several large cells within which the class probabilities are homogeneous. We exploit this property to obtain the mutual information between a multi-class OctoMap and a single ray as a summation over a subset of OcTree leaf nodes instead of individual map cells. This simplification provides a significant performance gain with no loss of accuracy. The following formulation can be considered a multi-class generalization of the run-length encoding technique introduced by \cite{fsmi}, using the mutual information lower bound in \eqref{eq:mut_inf_semantic} and the multi-class OcTree defined earlier in this section.

Suppose that the map cells along a single beam $\bfR_{\tau}\bfeta_b$ have piecewise-constant multi-class probabilities such that the set $\{m_i \mid i \in \calR_{\tau, b}(r_{max})\}$ can be partitioned into $Q_{\tau, b}$ groups of consecutive cells indexed by $\calR_{\tau, b}^q(r_{max})$, $q = 1, \ldots, Q_{\tau, b}$, where:
\begin{equation}
\begin{split}
    p_t(m_i = k) = p_t(m_j = k), \\ \forall i, j \in \calR_{\tau, b}^q(r_{max}), \;\;\; \forall k \in \calK.
\end{split}
\end{equation}
In this case, the log-odds probabilities encountered by a ray cast can be compressed using semantic run-length encoding, defined as below.

\begin{definition}
A \textit{semantic run-length encoding} (SRLE) of a ray $\bfR_\tau \bfeta_b$ cast through a multi-class OcTree is an ordered list of tuples of the form $[(\omega_{\tau,b,q}, \bfchi_{t,q})]^{Q_{\tau, b}}_{q=1}$, where $\omega_{\tau,b,q}$ and $\bfchi_{t,q}$ respectively represent the width and the log-odds vector of the intersection between the ray and the cells in $\calR_{\tau, b}^q(r_{max})$. The width $\omega_{\tau,b,q}$ is the number of OcTree \textit{elements} along the ray intersection, where an OcTree element is a cell with the smallest physical dimensions.
\end{definition}

\begin{figure}[t]
  \centering
  \includegraphics[width=\linewidth]{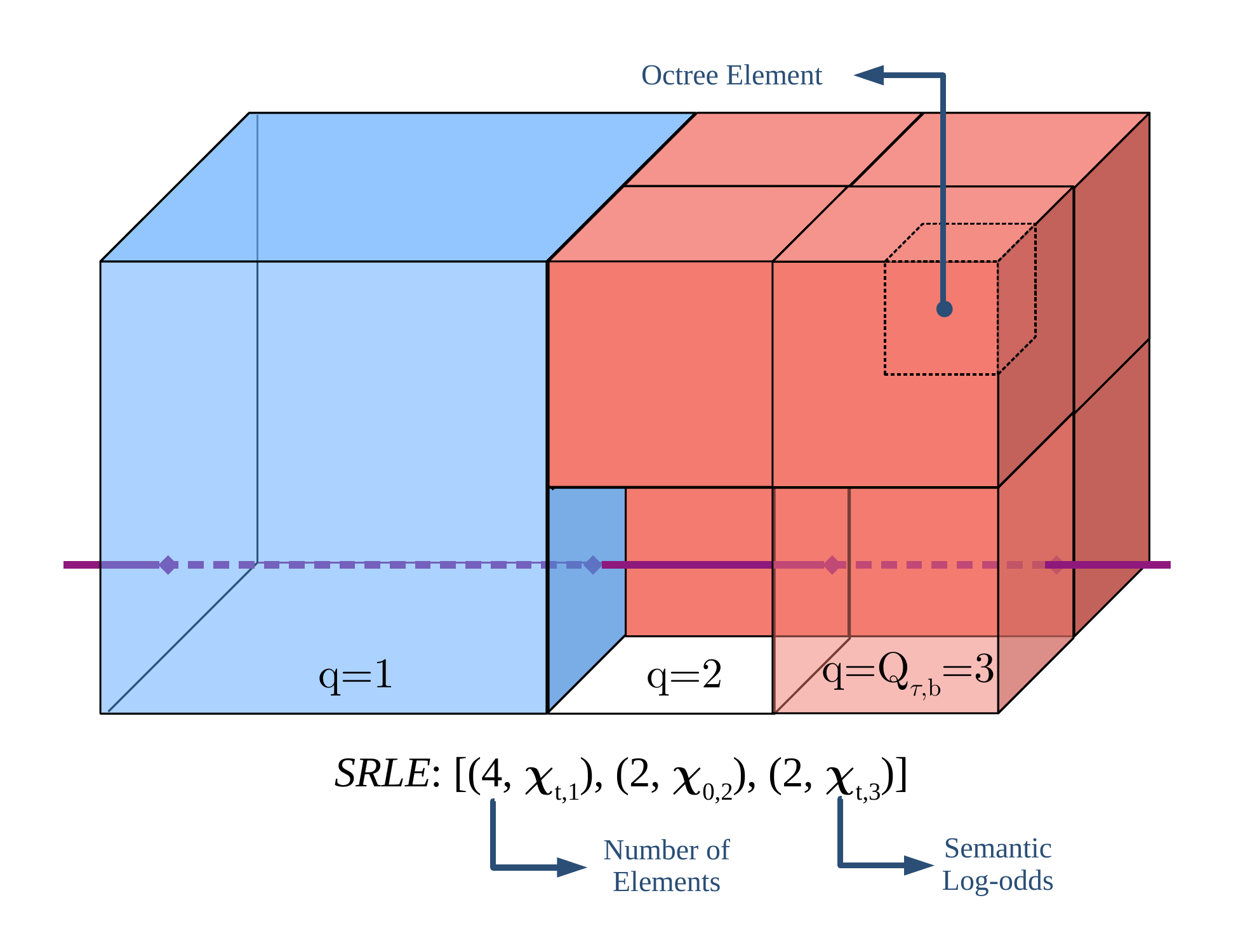}
  \caption{Ray cast representation as semantic run-length encoding (SRLE). The multi-class log-odds $\bfchi_{t,q}$ are uniform within each cube. The voxel corresponding to $q=2$ is unexplored, hence its multi-class log-odds are denoted as $\bfchi_{0,2}$.}
\label{fig:SRLE}
\end{figure}

Fig.~\ref{fig:SRLE} shows an example of SRLE over a semantic OctoMap. While SRLE can be used in a uniform-resolution grid map, it is particularly effective of a multi-class OcTree, which inherently contains large regions with homogeneous multi-class log-odds. Additionally, the OcTree data structure allows faster ray casting since it can be done over OcTree leaf nodes \cite{ray_trace_octomap_1, ray_trace_octomap_2}, instead of a uniform-resolution grid as in \cite{bresenham}. 

SRLE ray casting delivers substantial gains in efficiency for mutual information computation since the contribution of each group $\{m_i \mid i \in \calR_{\tau, b}^q(r_{max})\}$ in the innermost summation of \eqref{eq:mut_inf_semantic} can be obtained in closed form. 

\begin{proposition}
\label{prop:mut_inf_semantic_octomap}
The Shannon mutual information between a single range-category measurement $\bfz_{\tau, b}$ and a semantic OctoMap $\bfm$ can be computed as:
\begin{equation}
\label{eq:mut_inf_semantic_octomap_single}
    I(\bfm; \bfz_{\tau, b} | \calZ_{1:t}) = \sum_{k=1}^{K} \sum_{q=1}^{Q_{\tau, b}} \rho_{\tau, b}(q,k) \Theta_{\tau, b}(q,k)
\end{equation}
where $Q_{\tau, b}$ is the number of partitions along the ray $\bfR_\tau \bfeta_b$ that have identical multi-class log-odds and the multi-class probabilities for each partition are denoted as:
\begin{equation*}
\begin{aligned}
    \begin{cases} \pi_t(q,k) = p_t(m_i = k) \\ \bfchi_{t,q} = \bfh_{t,i} \end{cases} \;\; \forall i \in \calR_{\tau,b}^q(r_{max}).
\end{aligned}
\end{equation*}
Furthermore, defining $\omega_{\tau,b,q} = | \calR_{\tau,b}^q(r_{max}) |$ as the width of $q$-th partition, we have:
\begin{equation*}
\begin{aligned}
    \rho_{\tau, b}(q,k) &:= \pi_t(q,k) \prod_{j=1}^{q-1} \pi_t^{\omega_{\tau,b,j}}(j,0),\\
    \Theta_{\tau, b}(q,k) &:= \beta_{\tau, b}(q,k) \frac{1 - \pi_t^{\omega_{\tau,b,q}}(q,0)}{1 - \pi_t(q,0)} + \\ &\frac{f(\bfphi^--\bfchi_{0,q}, \bfchi_{t,q})}{(1 - \pi_t(q,0))^2} \Big[(\omega_{\tau,b,q} - 1) \pi_t^{\omega_{\tau,b,q} + 1}(q,0) \\ &\qquad\qquad\qquad - \omega_{\tau,b,q} \pi_t^{\omega_{\tau,b,q}}(q,0) + \pi_t(q,0)\Big],\\
    \beta_{\tau, b}(q,k) &:= f(\bfphi^+ + \bfE_{k+1}\bfpsi^+ - \bfchi_{0,q}, \bfchi_{t,q}) \\ &+ \sum_{j=1}^{q-1} \omega_{\tau,b,j} f(\bfphi^--\bfchi_{0,j}, \bfchi_{t,j}).
\end{aligned}
\end{equation*}
\end{proposition}

\begin{proof}
See Appendix~\ref{app:mut-inf-semantic-octomap}.
\end{proof}

In \eqref{eq:mut_inf_semantic_octomap_single}, $\rho_{\tau, b}(q, k)$ relates to the event that the partition $\calR_{\tau,b}^q(r_{max})$ belongs to category $k$ while all of the previous partitions along the ray $\bfR_\tau \bfeta_b$ are free. Analogous to the definition of $C_{\tau, b}(n, k)$ in Proposition~\ref{prop:mut_inf_semantic}, $\beta_{\tau, b}(q,k)$ is the weighted sum of log-ratios $f(\bfphi, \bfchi)$ for the first $q$ partitions along the ray $\bfR_\tau \bfeta_b$ under the same event as the one $\rho_{\tau, b}(q, k)$ is associated with. Accumulating the multi-class probabilities within the partition $\calR_{\tau,b}^q(r_{max})$ yields $\Theta_{\tau, b}(q,k)$, see \eqref{eq:within_section_sum_res} for more details. Therefore, the mutual information in \eqref{eq:mut_inf_semantic_octomap_single} is equivalent to the expectation of accumulated log-ratios $\Theta_{\tau, b}(q,k)$ over all possible instantiations of $\bfz_{\tau, b}$.

Proposition~\ref{prop:mut_inf_semantic_octomap} allows an extension of the mutual-information lower bound in Proposition~\ref{prop:mut_inf_semantic} to semantic OctoMap representations, summarized in the corollary below. The proof follows directly from the additive property of mutual information between a semantic OctoMap and a sequence of independent observations.

\begin{corollary}
Given a sequence of range-category observations $\calZ_{t+1:t+T}$, the Shannon mutual information between $\calZ_{t+1:t+T}$ and a semantic OctoMap $\bfm$ can be lower bounded as:
\begin{equation}
\label{eq:mut_inf_semantic_octomap}
\begin{aligned}
I(\bfm; &\calZ_{t+1:t+T} | \calZ_{1:t}) \geq I\prl{\bfm; \underline{\calZ}_{t+1:t+T} | \calZ_{1:t}}\\
&=\sum_{\tau=t+1}^{t+T} \sum_{b=1}^{B} \sum_{k=1}^K \sum_{q=1}^{Q_{\tau,b}} \rho_{\tau, b}(q,k) \Theta_{\tau, b}(q,k),
\end{aligned}
\end{equation}
where $\underline{\calZ}_{t+1:t+T}$ is a subset of non-overlapping measurements that satisfy \eqref{eq:nonoverlapping}, and $\rho_{\tau, b}(q,k)$ and $\Theta_{\tau, b}(q,k)$ are defined in Proposition~\ref{prop:mut_inf_semantic_octomap}.
\end{corollary}

The same approach as in Alg.~\ref{alg:sem_exp} is used for autonomous exploration over a semantic OctoMap. However, we employ the information computation formula of \eqref{eq:mut_inf_semantic_octomap} to quantify the informativeness of candidate robot trajectories. The active mapping method in Alg.~\ref{alg:sem_exp} provides a greedy exploration strategy, which does not change subsequent control inputs based on the updated map distribution. Greedy exploration may be sub-optimal and manifests itself as back and forth travel between map frontiers. We alleviate this behavior by (a) computing the information along the whole trajectory as opposed only at the frontiers or next best view, and (b) re-plan frequently to account for the updated map distribution. Discounted by distance traveled as the cost of a trajectory, this leads to a more accurate calculation of information gain along a candidate path which rules out most of the back and forth visiting behaviour. It is also important to mention that the main scope of this work is introduction of a novel multi-class semantic OcTree representation and the mutual information between such model and range-category observations. Our method enables fast and accurate evaluation of information for any set of candidate trajectories, likes of which can be generated by random tree methods \cite{fsmi, info_exp_8} or hierarchical planning strategies \cite{tare} or, in the simplest form, a greedy approach that computes paths to each frontier. We believe utilizing our proposed information measure to score candidate viewpoints would be complementary, rather than an alternative, to the state-of-the-art exploration methods that use sophisticated optimization strategies \cite{tare_2, tare, arash_iros22}.

\begin{figure}[t]
  \includegraphics[width=0.48\linewidth]{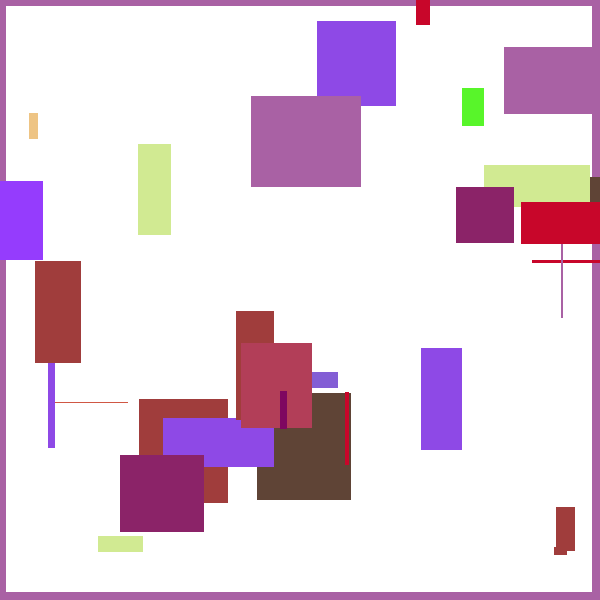}%
  \hfill%
  \includegraphics[width=0.48\linewidth]{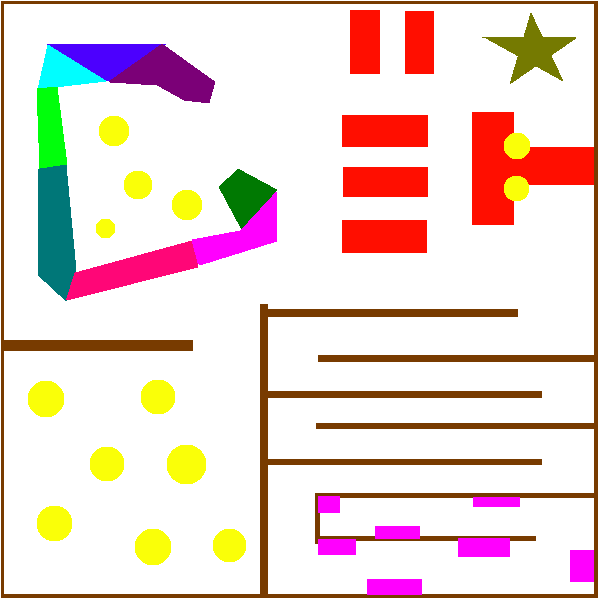}
  \caption{Synthetic environments used for comparisons among frontier-based exploration \cite{frontier}, FSMI \cite{fsmi}, and SSMI. Different semantic categories are represented by distinct colors. Left: An instance of procedurally generated random environment with $10$ object classes. Right: Hand-designed environment with corridor and block structures with $12$ object classes.}
  \label{fig:2d_true_map}
\end{figure}

\begin{figure*}[t]
    \captionsetup[subfigure]{justification=centering}
    \begin{subfigure}[t]{0.2575\linewidth}
    \includegraphics[width=\linewidth]{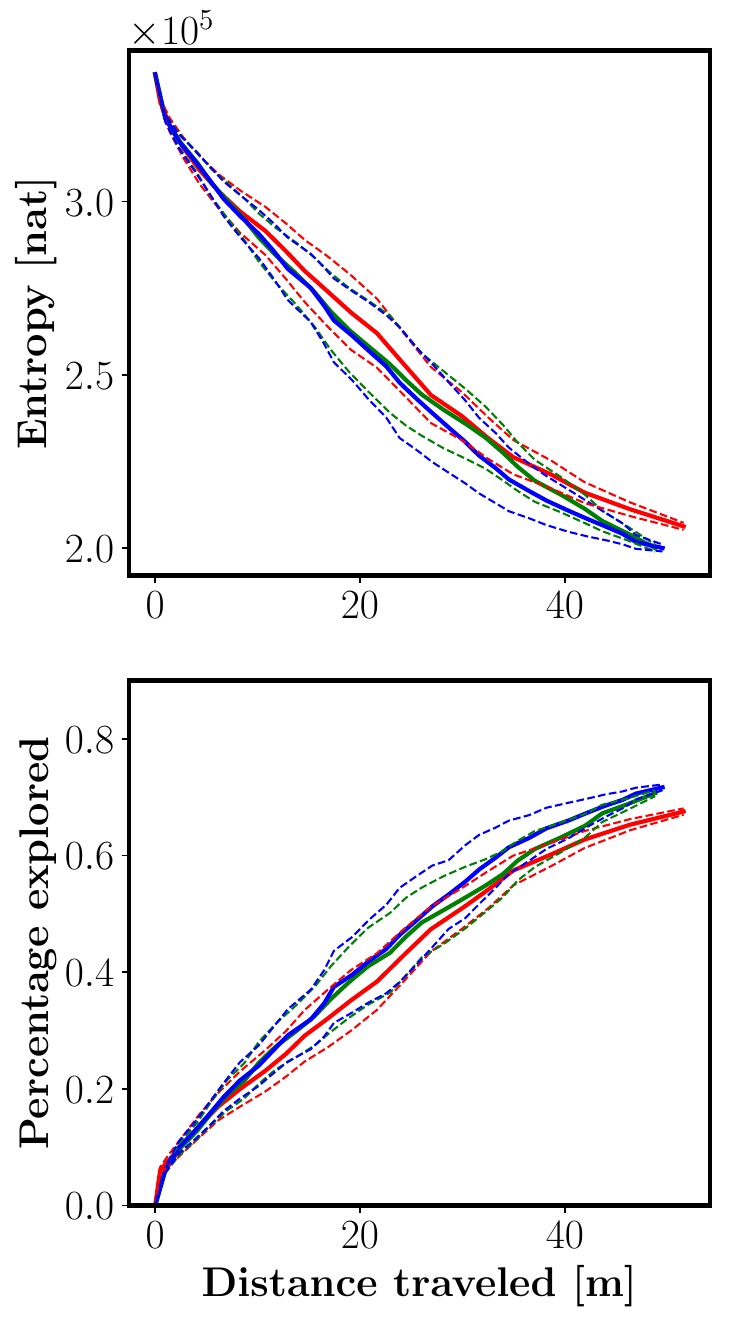}
    \caption{Random map / Binary exploration}
    \label{fig:2d_exp_results_a}
    \end{subfigure}%
    \hfill%
    \begin{subfigure}[t]{0.24\linewidth}
    \includegraphics[width=\linewidth]{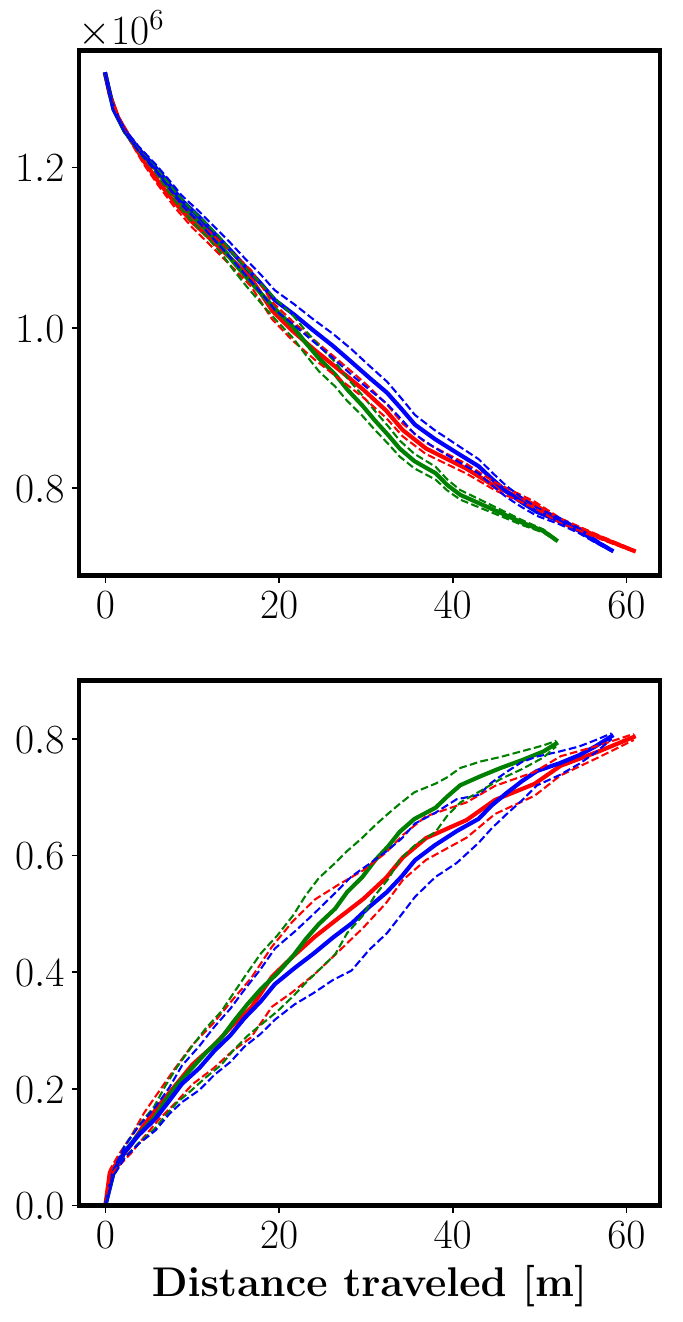}
    \caption{Random map / Semantic exploration}
    \label{fig:2d_exp_results_b}
    \end{subfigure}%
    \hfill%
    \begin{subfigure}[t]{0.24\linewidth}
    \includegraphics[width=\linewidth]{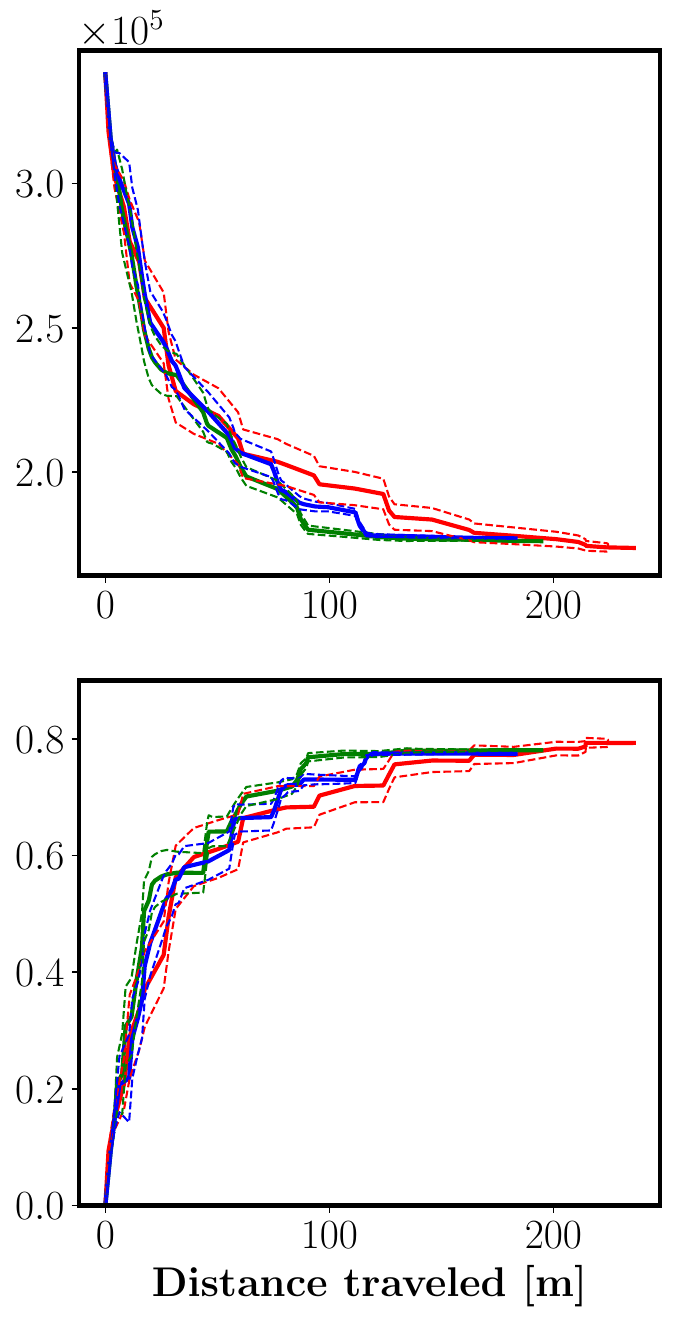}
    \caption{Structured map / Binary exploration}
    \label{fig:2d_exp_results_c}
    \end{subfigure}%
    \hfill%
    \begin{subfigure}[t]{0.24\linewidth}
    \includegraphics[width=\linewidth]{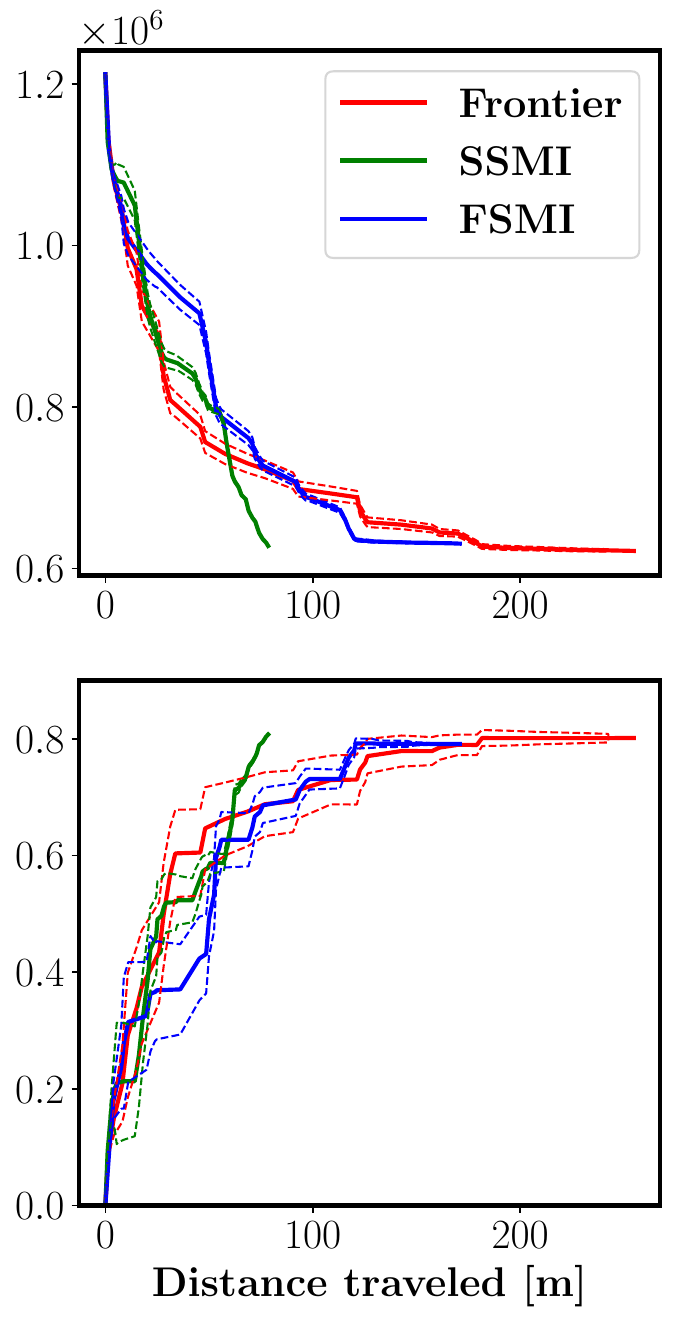}
    \caption{Structured map / Semantic exploration}
    \label{fig:2d_exp_results_d}
    \end{subfigure}%
    \caption{Simulation results for active mapping on the environments in Fig.~\ref{fig:2d_true_map}, for $20$ exploration iterations. Solid and dotted lines represent mean and $1$ standard deviation from the mean, respectively. (a), (b): Exploration performance averaged over $10$ random environments with $3$ random starting positions for each instance. Exploration in the random environments sometimes did not terminate before the maximum number of iterations, and therefore the corresponding curves do not flatten. This can be attributed to the fact that random maps, with the same size as the structured map, contain more frontiers that need to be explored in each iteration since each scan has a higher probability of being occluded in multiple angles due to the lack of certain patterns such as corridors. (c), (d): Exploration performance on the structured environment averaged over $3$ random starting positions. For the structured map, the exploration terminates before reaching the maximum number of iterations, which explains the flat curves at the end of the corresponding plots.}
    \label{fig:2d_exp_results}
\end{figure*}

\subsection{Computational Complexity}

Note that the mutual information computations in both \eqref{eq:mut_inf_semantic} and \eqref{eq:mut_inf_semantic_octomap} can be performed recursively. For \eqref{eq:mut_inf_semantic}, we have:
\begin{align}
p_{\tau, b}(n+1, k) &= p_{\tau, b}(n, k) \frac{p_{t}(m_{j^*_{\tau,b}}=k) p_{t}(m_{i^*_{\tau,b}}=0)}{p_{t}(m_{i^*_{\tau,b}}=k)},\notag\\
C_{\tau, b}(n+1, k) &= C_{\tau, b}(n, k) \notag\\ 
-& f(\bfphi^+ + \bfE_{k+1}\bfpsi^+ - \bfh_{0,i_{\tau,b}^*}, \bfh_{t,i_{\tau,b}^*}) \label{eq:p_C_recurse}\\ 
+& f(\bfphi^+ + \bfE_{k+1}\bfpsi^+ - \bfh_{0,j_{\tau,b}^*}, \bfh_{t,j_{\tau,b}^*}) \notag\\ 
+& f(\bfphi^--\bfh_{0,i_{\tau,b}^*}, \bfh_{t,i_{\tau,b}^*}),\notag
\end{align}
where $j_{\tau,b}^*$ and $i^*_{\tau,b}$ correspond to the index of farthest map cell in $\Tilde{\calR}_{\tau,b}(n+1)$ and $\Tilde{\calR}_{\tau,b}(n)$, respectively. A similar recursive pattern can be found in \eqref{eq:mut_inf_semantic_octomap}:
\begin{equation}
\begin{aligned}
\rho_{\tau, b}(q+1, k) &= \rho_{\tau, b}(q, k) \frac{\pi_{t}(q+1,k) \pi^{\omega_{\tau,b,q}}_{t}(q,0)}{\pi_{t}(q,k)},\\
\beta_{\tau, b}(q+1, k) &= \beta_{\tau, b}(q, k) \\ &- f(\bfphi^+ + \bfE_{k+1}\bfpsi^+ - \bfchi_{0,q}, \bfchi_{t,q}) \\ &+ f(\bfphi^+ + \bfE_{k+1}\bfpsi^+ - \bfchi_{0,q+1}, \bfchi_{t,q+1}) \\ &+ \omega_{\tau, b, q} f(\bfphi^--\bfchi_{0,q}, \bfchi_{t,q}).
\end{aligned}
\label{eq:}
\end{equation}
This implies that the innermost summations of \eqref{eq:mut_inf_semantic} and \eqref{eq:mut_inf_semantic_octomap} can be obtained in $O(N_{\tau,b})$ and $O(Q_{\tau,b})$, respectively, where $N_{\tau,b}$ is the number of map cells along a single ray $\bfR_{\tau}\bfeta_b$ up to its maximum range, and $Q_{\tau,b}$ is the number of groups of consecutive cells that possess the same multi-class probabilities. In an environment containing $K$ object classes, evaluating the informativeness of a trajectory composed of $T$ observations, where each observation contains $B$ beams, has a complexity of $O(T B K N_{\tau,b})$ for a regular-grid multi-class representation and a complexity of $O(T B K Q_{\tau,b})$ for a multi-class OcTree representation.

As we demonstrate in Sec.~\ref{subsec:SRLE}, for a ray $\bfR_{\tau}\bfeta_b$ we often observe that $Q_{\tau,b}$ is significantly smaller than $N_{\tau,b}$ thanks to the OcTree pruning mechanism. Since $N_{\tau,b}$ scales linearly with the map resolution, the complexity of information computation over a semantic OctoMap grows sub-linearly with respect to the inverse of the OcTree element dimensions, which is a parameter analogous to the map resolution.

%% file: tex/experiments.tex
\section{Experiments}
\label{sec:experiments}

\begin{figure}[t]
    \captionsetup[subfigure]{justification=centering}
    \begin{subfigure}[t]{0.48\linewidth}
    \includegraphics[width=\linewidth]{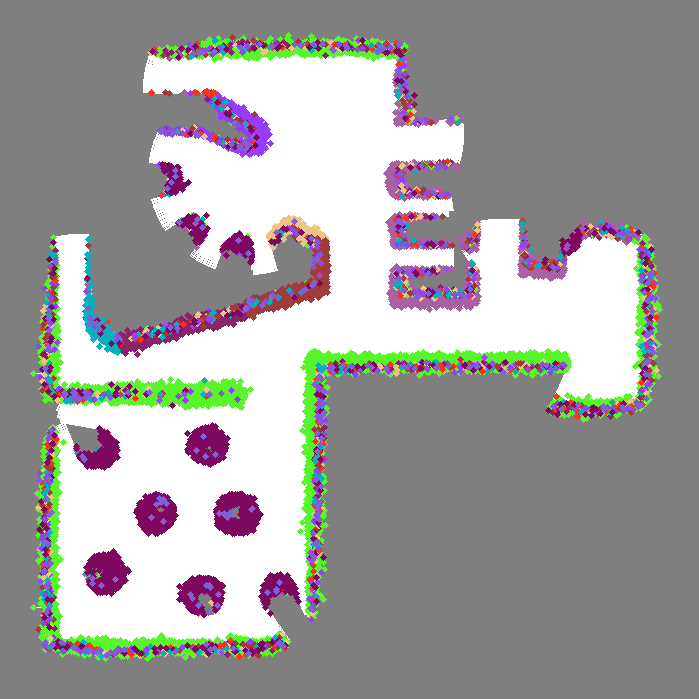}
    \caption{Partially explored semantic map}
    \label{fig:2d_info_surface_a}
    \end{subfigure}%
    \hfill%
    \begin{subfigure}[t]{0.48\linewidth}
    {%
    \setlength{\fboxsep}{0pt}%
    \setlength{\fboxrule}{0.01pt}%
    \fbox{\includegraphics[width=\linewidth]{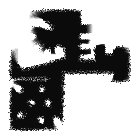}}%
    }%
    \caption{Entropy surface}
    \vspace{0.1\textwidth}%
    \label{fig:2d_info_surface_b}
    \end{subfigure}\\
    \begin{subfigure}[t]{0.48\linewidth}
    \includegraphics[width=\linewidth]{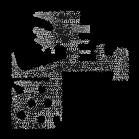}
    \caption{Occupancy mutual information surface}
    \label{fig:2d_info_surface_c}
    \end{subfigure}%
    \hfill%
    \begin{subfigure}[t]{0.48\linewidth}
    \includegraphics[width=\linewidth]{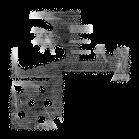}
    \caption{Multi-class mutual information surface}
    \label{fig:2d_info_surface_d}
    \end{subfigure}
    \caption{Comparison between different mutual information formulations used for exploration. (a), (b): A snapshot of 2-D exploration showing the map estimate and the corresponding uncertainty, where the entropy for each pixel $i$ is computed as ~\eqref{eq:pixel_entropy}. (c), (d): The mutual information used to find the informative trajectory by FSMI and SSMI, respectively. Brighter pixels indicate larger values.}
    \label{fig:2d_info_surface}
\end{figure}

In this section, we evaluate the performance of SSMI in simulated and real-world experiments. We compare SSMI with two baseline exploration strategies, i.e., frontier-based exploration~\cite{frontier} and FSMI~\cite{fsmi}, in a 2-D active binary mapping scenario in Sec.~\ref{subsec:exp_2d_bin} and a 2-D active multi-class mapping scenario in Sec.~\ref{subsec:exp_2d_multi}. All three methods use our range-category sensor model in \eqref{eq:log_inverse_observation_model} and our Bayesian multi-class mapping in \eqref{eq:log-odds-bayes-rule} but select informative robot trajectories $\bfX_{t+1:t+T}(\bfu_{t:t+T-1})$ based on their own criteria. In Sec.~\ref{subsec:SRLE}, we evaluate the improvement in ray tracing resulting from SRLE through an experiment in a 3-D simulated Unity environment. In Sec.~\ref{subsec:octomap_class_comp}, we investigate the influence of the number of stored semantic classes on mapping performance. Additionally, in Sec.~\ref{subsec:3-D_exp_sim}, we use a similar 3-D simulation environment to apply SSMI alongside Frontier, FSMI, and hierarchical coverage maximization method TARE~\cite{tare}. In this section we use our OcTree-based multi-class information computation introduced in Sec.~\ref{sec:info_comp_octomap} in order to demonstrate large-scale realistic active multi-class mapping. Finally, in Sec.~\ref{subsec:real_world_map} and Sec.~\ref{subsec:real_world_exp}, we test SSMI mapping and exploration in real environments using ground wheeled robots. An open-source implementation of SSMI is available on GitHub\footnote{\url{https://github.com/ExistentialRobotics/SSMI}.}.

In each planning step of 2-D exploration, we identify frontiers by applying edge detection on the most likely map at time $t$ (the mode of $p_t(\bfm)$). Then, we cluster the edge cells by detecting the connected components of the boundaries between explored and unexplored space. We plan a path from the robot pose $\bfX_t$ to the center of each frontier using $A^*$ graph search and provide the path to a low-level controller to generate $\bfu_{t:t+T-1}$. For 3-D exploration, we first derive a 2-D occupancy map by projecting the most likely semantic OctoMap at time $t$ onto the $z = 0$ surface and proceed with similar steps as in 2-D path planning.

\subsection{2-D Binary Exploration}
\label{subsec:exp_2d_bin}

We consider active binary occupancy mapping first. We compare SSMI against Frontier and FSMI in $1$ structured and $10$ procedurally generated 2-D environments, shown in Fig.~\ref{fig:2d_true_map}. A 2-D LiDAR sensor is simulated with additive Gaussian noise $\calN(0,0.1)$. Fig.~\ref{fig:2d_exp_results_a} and Fig.~\ref{fig:2d_exp_results_c} compare the exploration performance in terms of map entropy reduction and percentage of the map explored per distance traveled among the three methods. SSMI performs similarly to FSMI in that both achieve low map entropy by traversing significantly less distance compared to Frontier.

\subsection{2-D Multi-class Exploration}
\label{subsec:exp_2d_multi}

Next, we use the same 2-D environments in Fig.~\ref{fig:2d_true_map} but introduce range-category measurements. Range measurements are subject to additive Gaussian noise $\calN(0,0.1)$, while category measurements have a uniform misclassification probability of $0.35$. Fig.~\ref{fig:2d_exp_results_b} and Fig.~\ref{fig:2d_exp_results_d} compare the semantic exploration performance for all three strategies. SSMI reaches the same level of map entropy as FSMI and Frontier but traverses a noticeably shorter distance. This can be attributed to the fact that only SSMI distinguishes map cells whose occupancy probabilities are the same but their per-class probabilities differ from each other. To further illustrate this, we visualize the entropy and information surfaces used by FSMI and SSMI. Fig.~\ref{fig:2d_info_surface_a} shows a snapshot of semantic exploration while Fig.~\ref{fig:2d_info_surface_b} visualizes the entropy of each pixel $i$ computed as:
\begin{equation}
    H(m_i | \calZ_{1:t}) = - \sum_{k = 0}^K p_t(m_i = k) \log{p_t(m_i = k)},
\label{eq:pixel_entropy}
\end{equation}
where $\calZ_{1:t}$ denote realized observations until time $t$. The task of exploration can be regarded as minimizing the conditional entropy summed over all pixels, i.e., map entropy. However, since the observations are not known in advance, we resort to estimate the reduction in uncertainty by computing the expectation over the observations. Accounting for the prior uncertainty in map, we arrive at maximizing mutual information as our objective, which is related to entropy as follows:
\begin{equation}
    H(m_i) - \bbE_{\calZ_{1:t}}\{H(m_i | \calZ_{1:t})\} = I(m_i; \calZ_{1:t}).
\label{eq:pixel_entropy_to_info}
\end{equation}

Therefore, the exploration performance is highly dependent upon the mutual information formulation, since it directly dictates how the uncertainty is quantified. As shown in Fig.~\ref{fig:2d_info_surface_d} and resulted from capturing per-class uncertainties, semantic mutual information of SSMI, computed in \eqref{eq:mut_inf_semantic} provides a smoother and more accurate estimation of information-rich regions compared to the binary mutual information formula used by FSMI (equation (18) in \cite{fsmi}) shown in Fig.~\ref{fig:2d_info_surface_c}.

\begin{figure}[t]
  \centering
  \includegraphics[width=\linewidth]{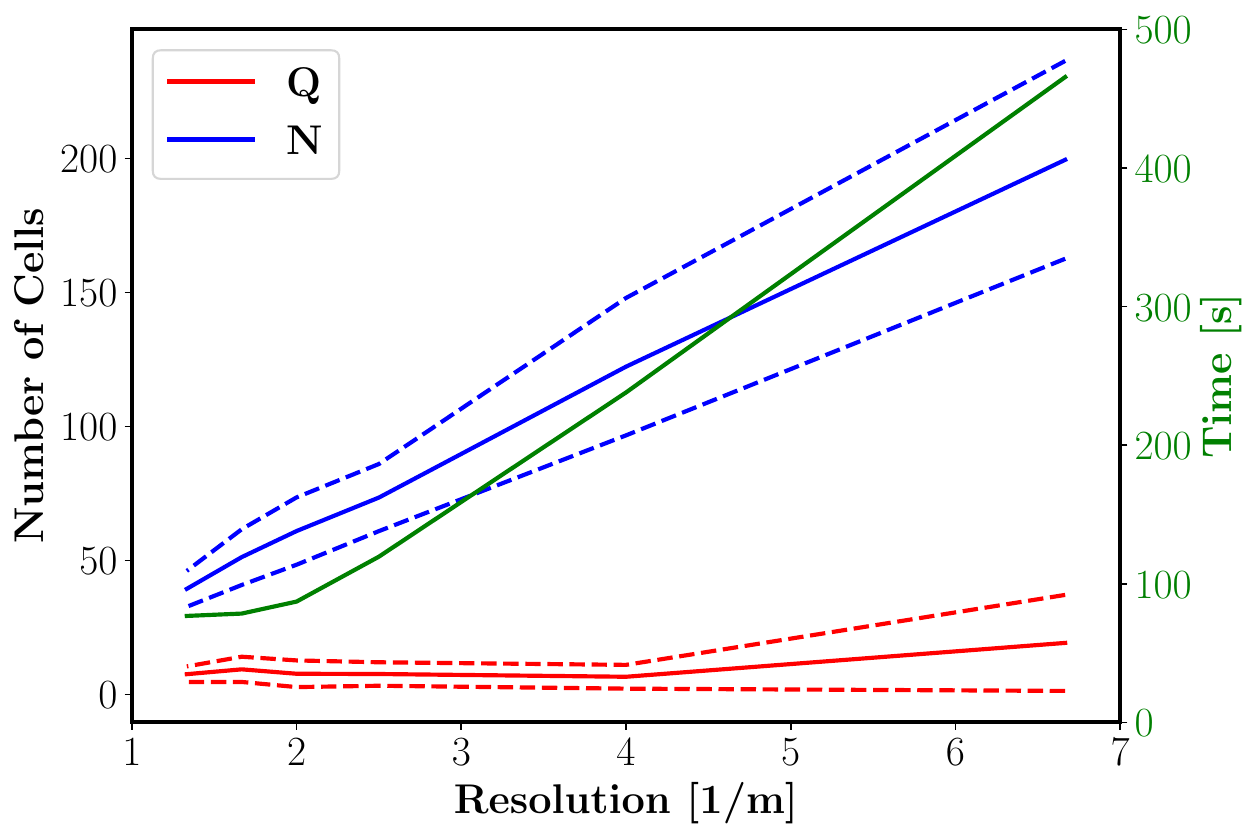}
  \caption{Variation of the visited OctoMap cells and OcTree elements denoted as $Q$ and $N$, respectively, with respect to the map resolution. Solid blue and red lines represent the average values for $Q$ and $N$ over all ray castings, while the dashed lines show one standard deviation from the average. The green curve shows the total exploration time for each map resolution. All measurements are accumulated in the course of $5$ exploration iterations.}
\label{fig:RLE_curve}
\end{figure}

\begin{figure*}[t]
    \begin{subfigure}[t]{0.32\linewidth}
    \includegraphics[width=\linewidth]{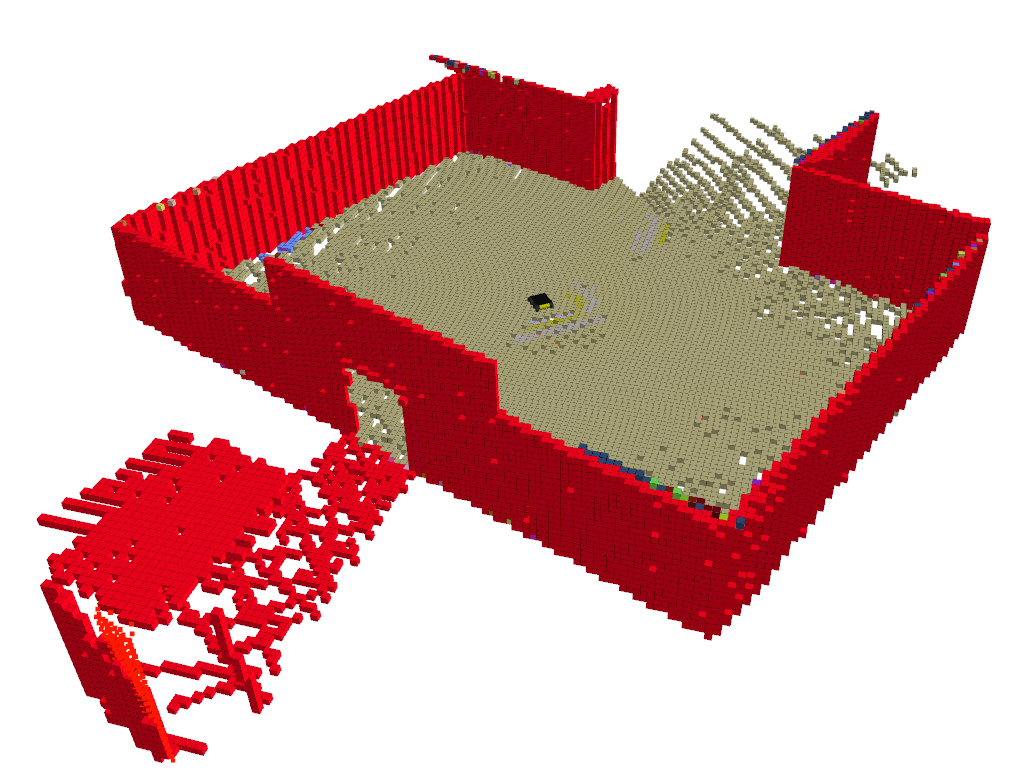}
    \captionsetup{justification=centering}
    \caption{The robot begins exploration.}
    \label{fig:3d_sim_env_a}
    \end{subfigure}%
    \hfill%
    \begin{subfigure}[t]{0.32\linewidth}
    \includegraphics[width=\linewidth]{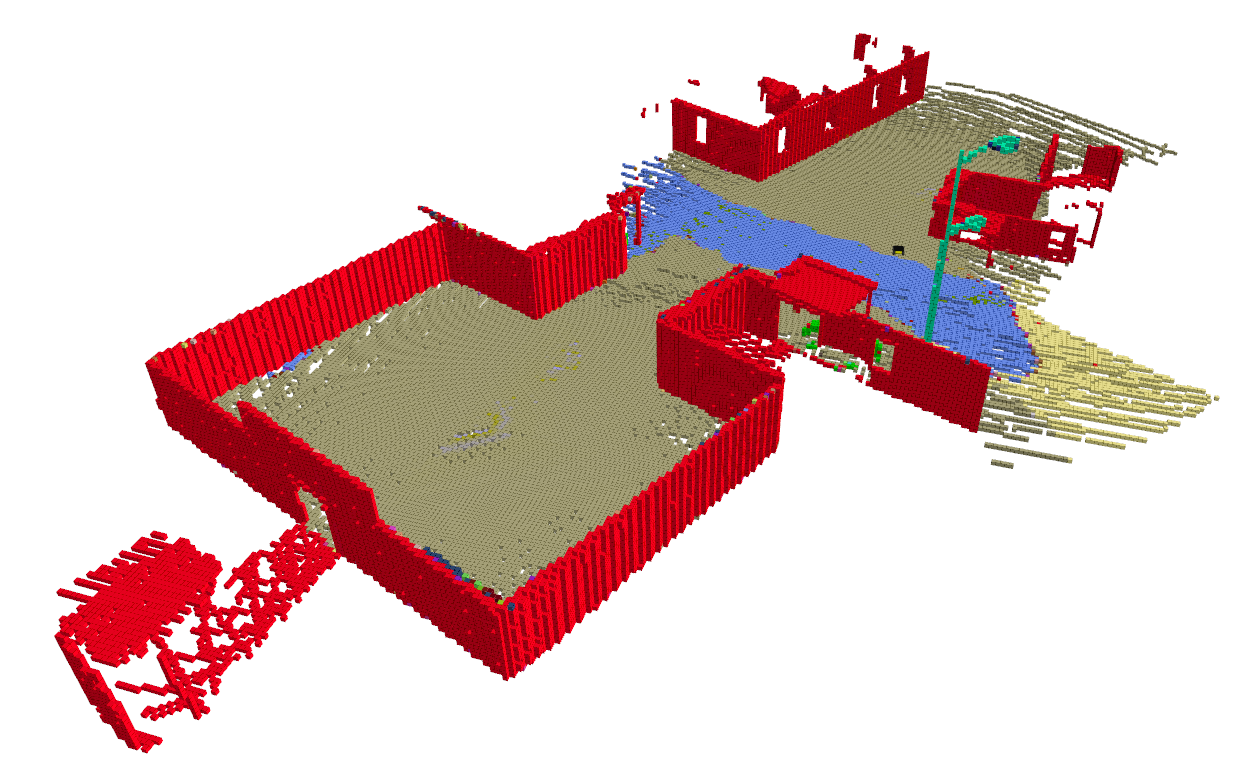}
    \caption{After $16$ iterations, the robot starts to refine previously explored areas to fill partially observed objects.}
    \label{fig:3d_sim_env_b}
    \end{subfigure}%
    \hfill%
    \begin{subfigure}[t]{0.32\linewidth}
    \includegraphics[width=\linewidth]{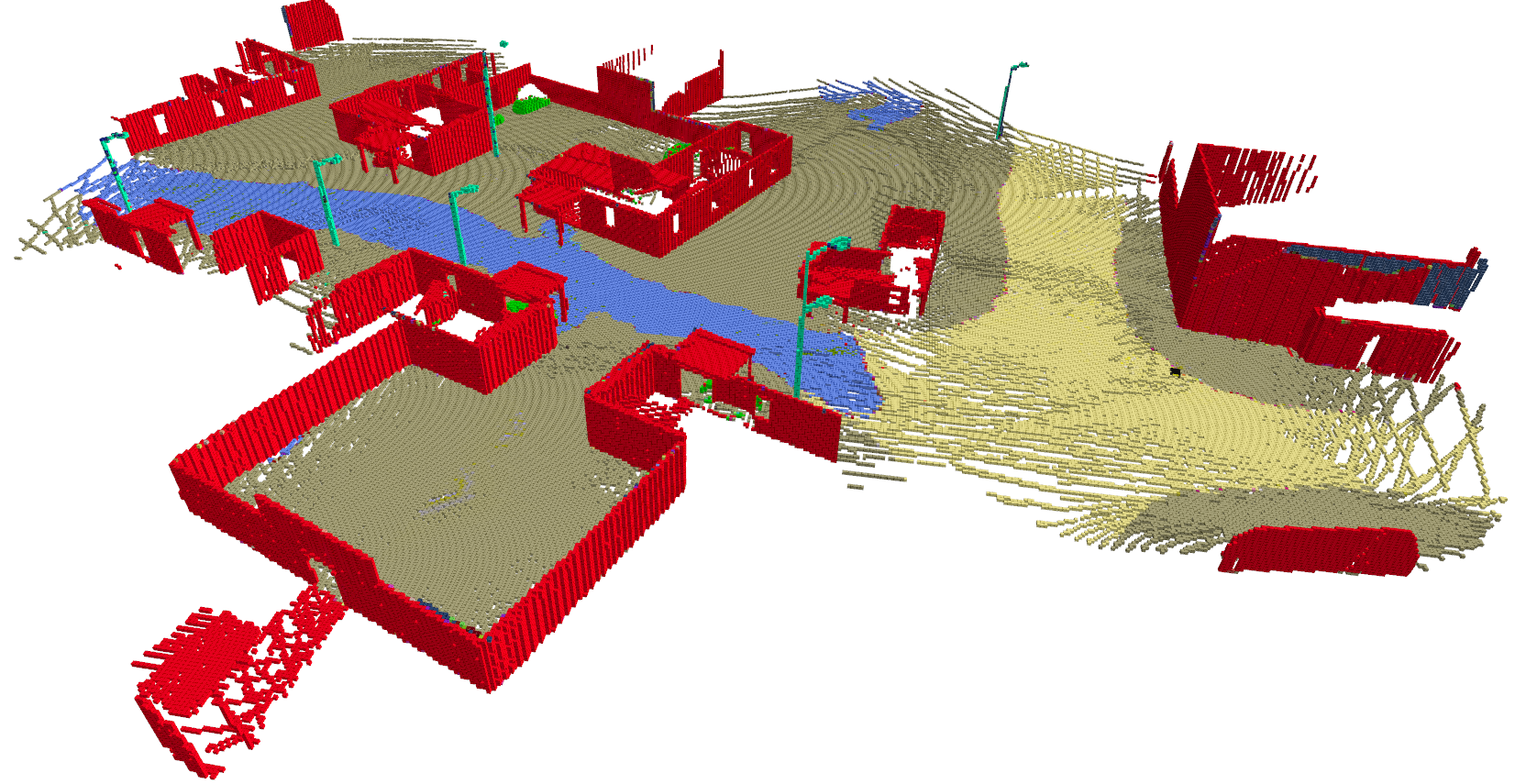}
    \caption{The robot explores unknown regions located on the boundaries of the explored area at iteration $40$.}
    \label{fig:3d_sim_env_c}
    \end{subfigure}\\
    \begin{subfigure}[t]{0.49\linewidth}
    \centering
    \includegraphics[width=\linewidth]{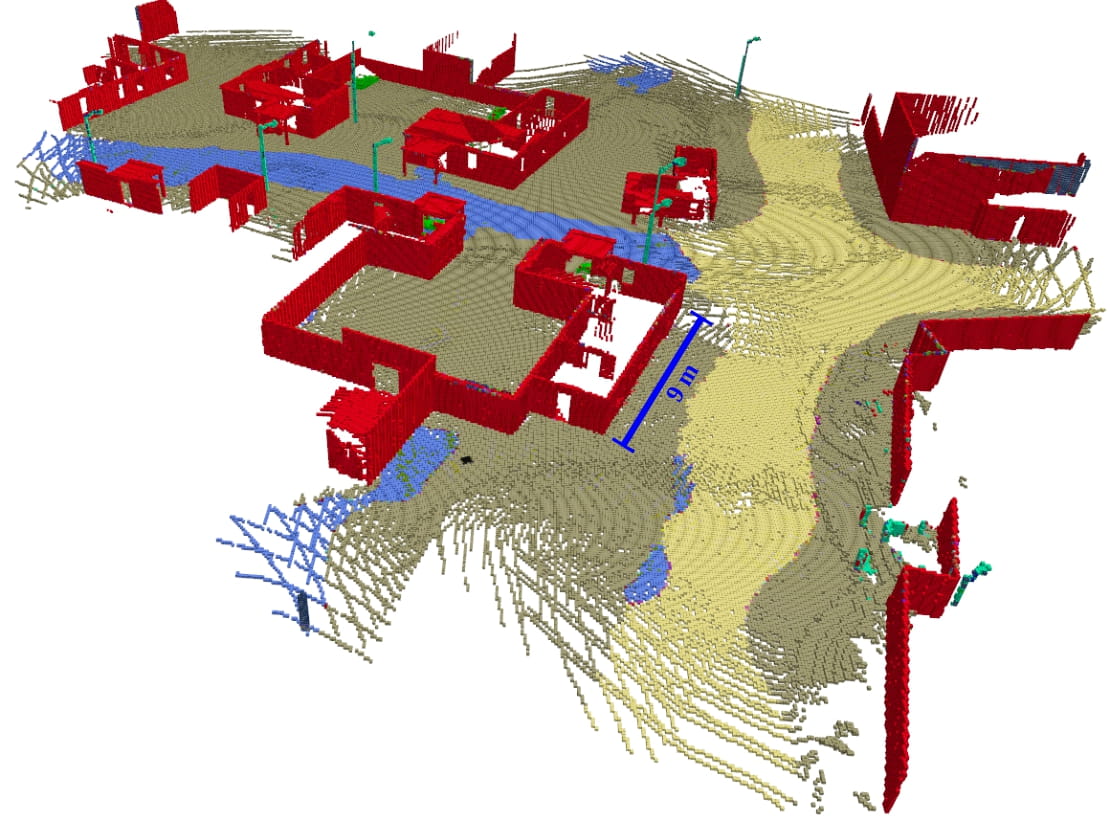}
    \captionsetup{justification=centering}
    \caption{Multi-class occupancy map after $60$ exploration iterations}
    \label{fig:3d_sim_env_d}
    \end{subfigure}%
    \hfill%
    \begin{subfigure}[t]{0.49\linewidth}
    \centering
    \includegraphics[width=\linewidth]{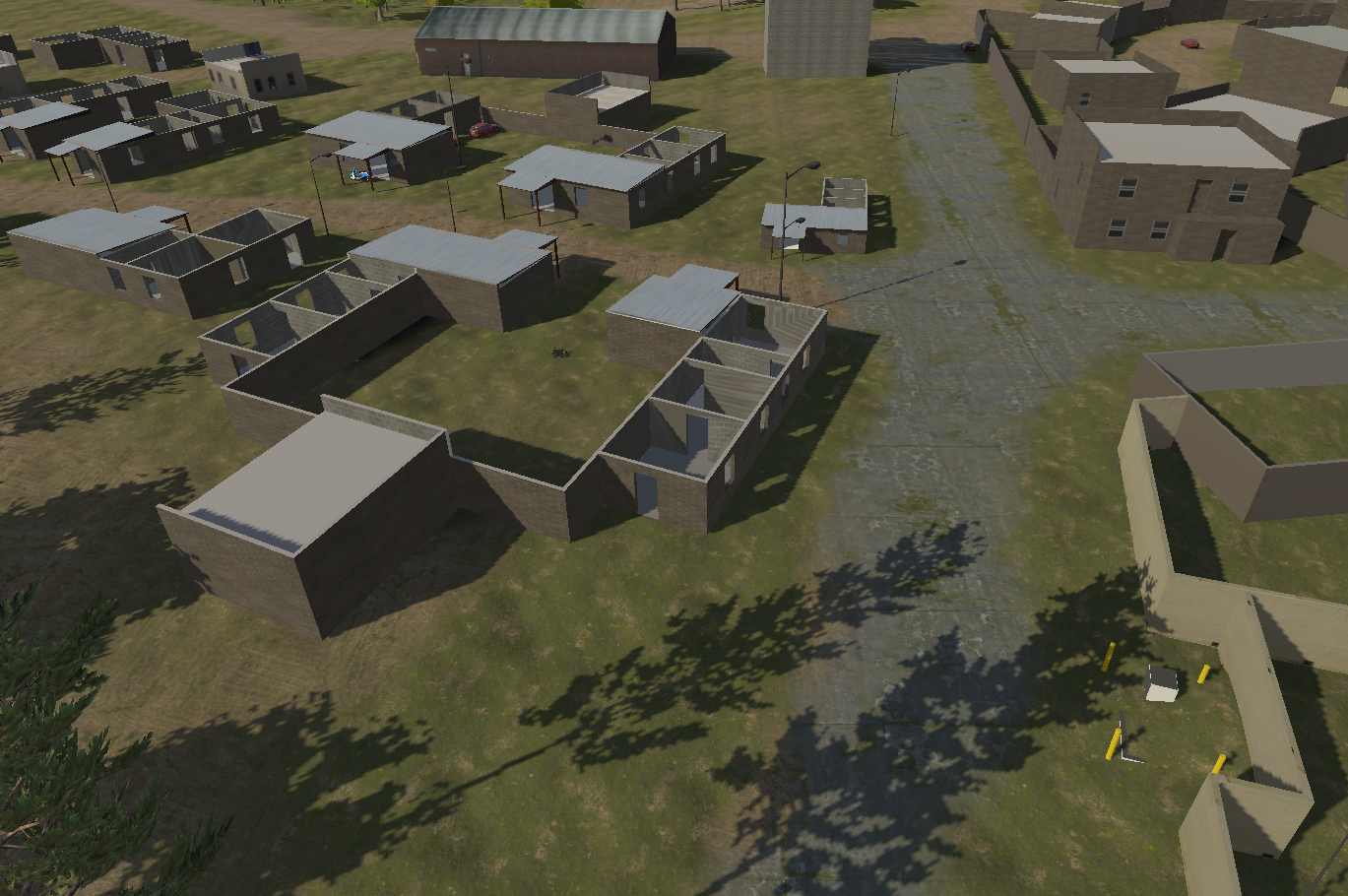}
    \captionsetup{justification=centering}
    \caption{Photo-realistic Unity simulation environment}
    \label{fig:3d_sim_env_e}
    \end{subfigure}
    \caption{Time lapse of autonomous exploration and multi-class mapping in a simulated Unity environment. The robot is equipped with an RGBD sensor and runs semantic segmentation. Different colors represent different semantic categories (grass, dirt road, building, etc.).}
    \label{fig:3d_sim_env}
\end{figure*}

\subsection{SRLE Compression for 3-D Ray Tracing}
\label{subsec:SRLE}

In this subsection, we evaluate the ray-tracing compression resulting from SRLE through an experiment in a photo-realistic 3-D Unity simulation, shown in Fig.~\ref{fig:3d_sim_env_e}. We use a Husky robot equipped with an RGBD camera and run a semantic segmentation algorithm over the RGB images. In order to remove irrelevant randomness, the sensors and the semantic segmentation are defined as error-free. We define \textit{map resolution} as the inverse of the dimensions of an OcTree element. For resolutions ranging from $1.3 m^{-1}$ to $6.6 m^{-1}$, we run $5$ exploration iterations using the semantic OctoMap and information computation of Sec.~\ref{sec:info_comp_octomap} and store all ray traces in SRLE format. Fig.~\ref{fig:RLE_curve} shows the change in distribution for the number of OctoMap cells $Q$ and OcTree elements $N$ visited during each ray trace, as well as the time required to execute each exploration episode as a function of map resolution. In other words, $N$ represents the number of cells to be processed during mapping and information computation as if the environment was represented as a regular 3-D grid, while $Q$ represents the actual number of processed semantic OctoMap cells. The pruning mechanism of the OcTree representation results in a substantial gain in terms of the number of cells visited for each ray tracing. As opposed to the almost linear growth of $N$, the distribution for $Q$ is effectively independent of the map resolution, except for very fine resolutions where void areas between observations rays prevent efficient pruning. However, for map resolutions larger than $2 m^{-1}$, the exploration time tends to grow larger with the increase of map resolution. This is attributed to the recursive ray insertion method of OctoMap in which it is required to re-compute log odds for each OcTree element along an observation ray whenever an observation ray does not carry the same (free or object class) state as the visited cell. In the subsequent 3-D experiments, we choose map resolution of $2 m^{-1}$ in order to balance between performance and map accuracy.

\begin{figure}[t]
  \centering
  \begin{subfigure}[t]{0.9\linewidth}
  \includegraphics[width=\linewidth]{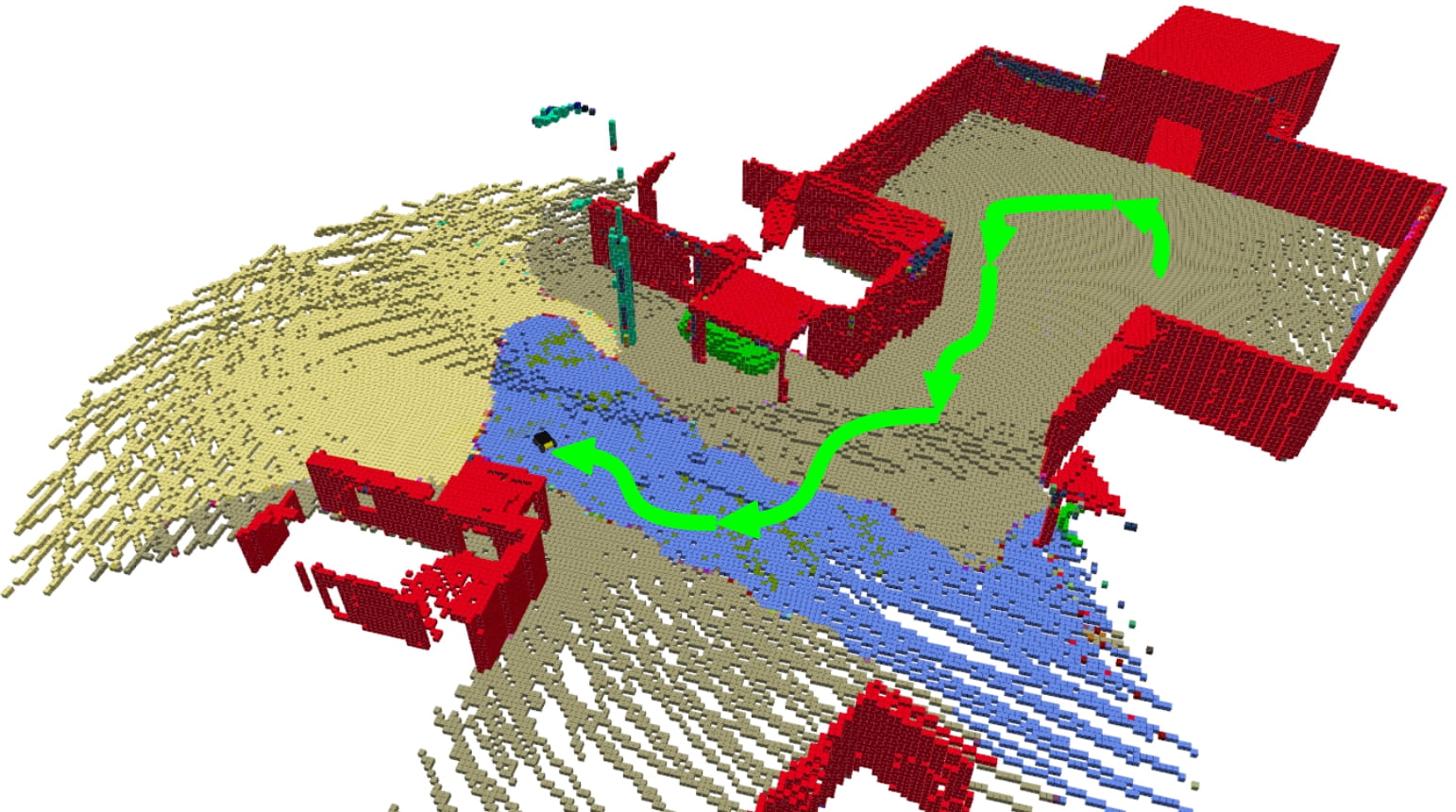}
  \label{fig:map_fps_path}
  \end{subfigure}\\
  \begin{subfigure}[t]{\linewidth}
  \includegraphics[width=\linewidth]{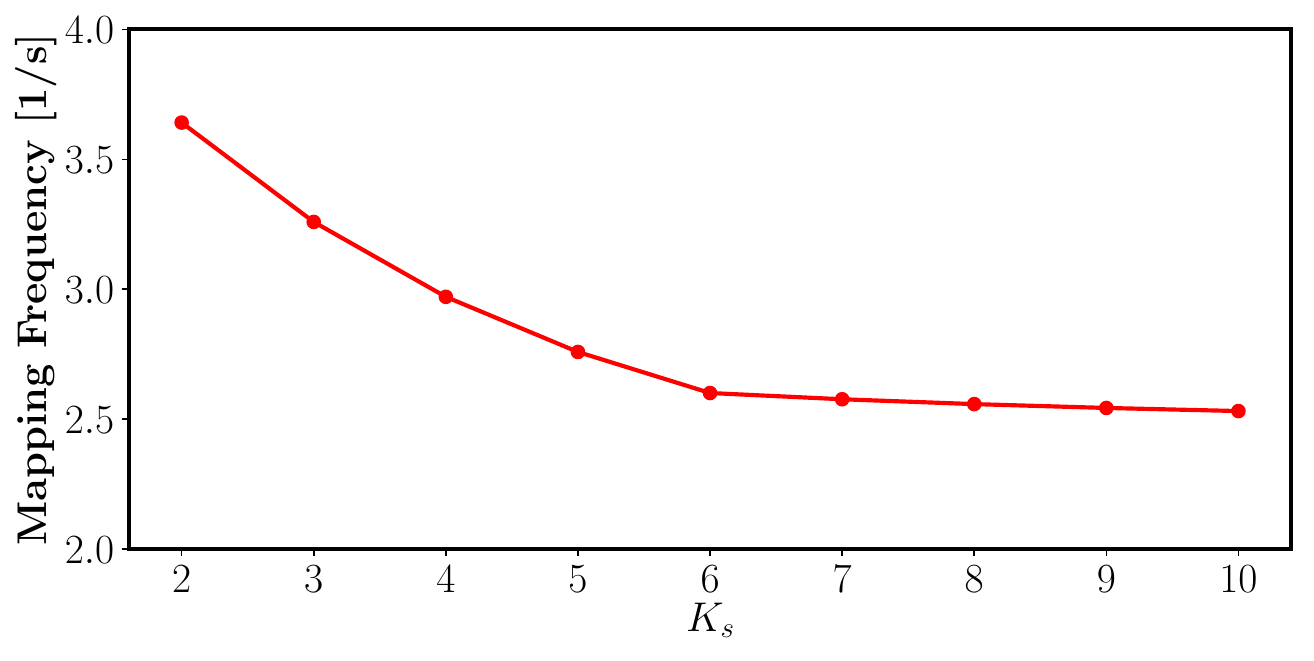}
  \label{fig:map_fps_vs_K_s}
  \end{subfigure}%
  \caption{Mapping time vs. number of stored classes. Top: Robot trajectory (in green) used for all mapping frequency evaluations. Bottom: Average mapping frequency as a function of the number of stored semantic classes $K_s$.}
\label{fig:map_fps}
\end{figure}

\subsection{Mapping Time vs. Number of Stored Classes}
\label{subsec:octomap_class_comp}

We analyse the influence of the number of stored classes in the semantic OctoMap on the mapping time. Let $K_s$ denote the number of stored semantic classes. Alg.~\ref{alg:fuse_obs} has $O(K_s)$ memory and $O(K_s \log{K_s})$ computational complexity (due to sorting in line \ref{alg:sorting}).
Furthermore, let $p_{\text{miss}}$ be the misclassification probability assumed to be uniformly distributed among all incorrect classes. Regarding accuracy, for a classifier with $p_{\text{miss}} < \frac{K-1}{K}$, where $K$ is the number of all object classes, the true class will be always asymptotically recoverable as long as $K_s \geq 2$, thanks to the auxiliary \textit{others} class that stores the accumulated probability of the $K - K_s$ least likely classes (see line \ref{alg:others} of Alg.~\ref{alg:fuse_obs}). In general, $K_s$ controls how fast the true class will be detected with the cost of additional memory use and computation. In order to quantitatively evaluate the effect of $K_s$ on mapping time, we consider the same Husky robot as the previous subsection with a fixed trajectory, shown in Fig.~\ref{fig:map_fps}~(top), and measure the mapping frequency as a function of $K_s$. Fig.~\ref{fig:map_fps}~(bottom) shows the decrease in average mapping frequency as $K_s$ increases. It is important to mention that the trajectory along which the data is collected only visits $6$ object classes, which explains the change in slope for $K_s > 6$.

\subsection{3-D Exploration in a Unity Simulation}
\label{subsec:3-D_exp_sim}

We evaluate SSMI in the same 3-D simulation environment as two previous subsections, however, this time the range measurements have an additive Gaussian noise of $\calN(0,0.1)$ and the semantic segmentation algorithm detects the true class with a probability of $0.95$ while the misclassification happens uniformly in the pixel space. Fig.~\ref{fig:3d_sim_env} shows several iterations of the exploration process. For comparison, we implemented a 3-D version of FSMI~\cite{fsmi} that utilizes run-length encoding to accelerate the information computation for a binary OctoMap. Moreover, we deploy the state-of-the-art hierarchical exploration method of TARE~\cite{tare} in our 3-D Unity simulation environment. Fig.~\ref{fig:3d_sim_res} shows the change in map entropy versus distance traveled and total elapsed time for all exploration strategies. We observe that SSMI is the most efficient in terms of solving the trade-off between path length and information gathered along the path. SSMI achieves the lowest entropy in the multi-class OctoMap. Similar to the discussion in Sec.~\ref{subsec:exp_2d_multi}, this observation can be ascribed to the fact that, among the compared methods, the only objective function which captures the uncertainty in both semantic classes and occupancy of the environment is the one used by SSMI. On the other hand, SSMI and FSMI require evaluation of mutual information along each candidate trajectory, which has the same cardinality as the number of all frontiers in the current map estimate $p_t(\bfm)$, whereas the hierarchical planning method employed by TARE only requires local trajectory computation with a global coverage path obtained at a coarse level. As a result, TARE exploration can be performed over a relatively shorter time period compared to SSMI and FSMI in scenarios where the number of frontiers is large, e.g. outdoor areas. Parallel computation of mutual information for each candidate trajectory or using heuristics such as frontier size in order to sort candidate solutions would improve the computation time of SSMI; however we believe these are outside of the scope of this paper. Fig.~\ref{fig:3d_sim_res_prec} compares the mapping precision of various object classes for the tested methods. SSMI exhibits higher precision for object categories that appear rarely, such as the \textit{Animal} or \textit{Tree} classes while Frontier slightly outperforms SSMI when it comes to mapping the \textit{Grass} and \textit{Dirt Road} categories. This can be explained by the tendency of SSMI towards achieving high overall classification precision even if it requires slight reduction of precision for certain object categories. Furthermore, TARE achieves the highest precision for the \textit{Building} class, which can be justified by the observation that the computed global coverage path tends to traverse near building walls.

\begin{figure}[t]
  \centering
  \includegraphics[width=\linewidth]{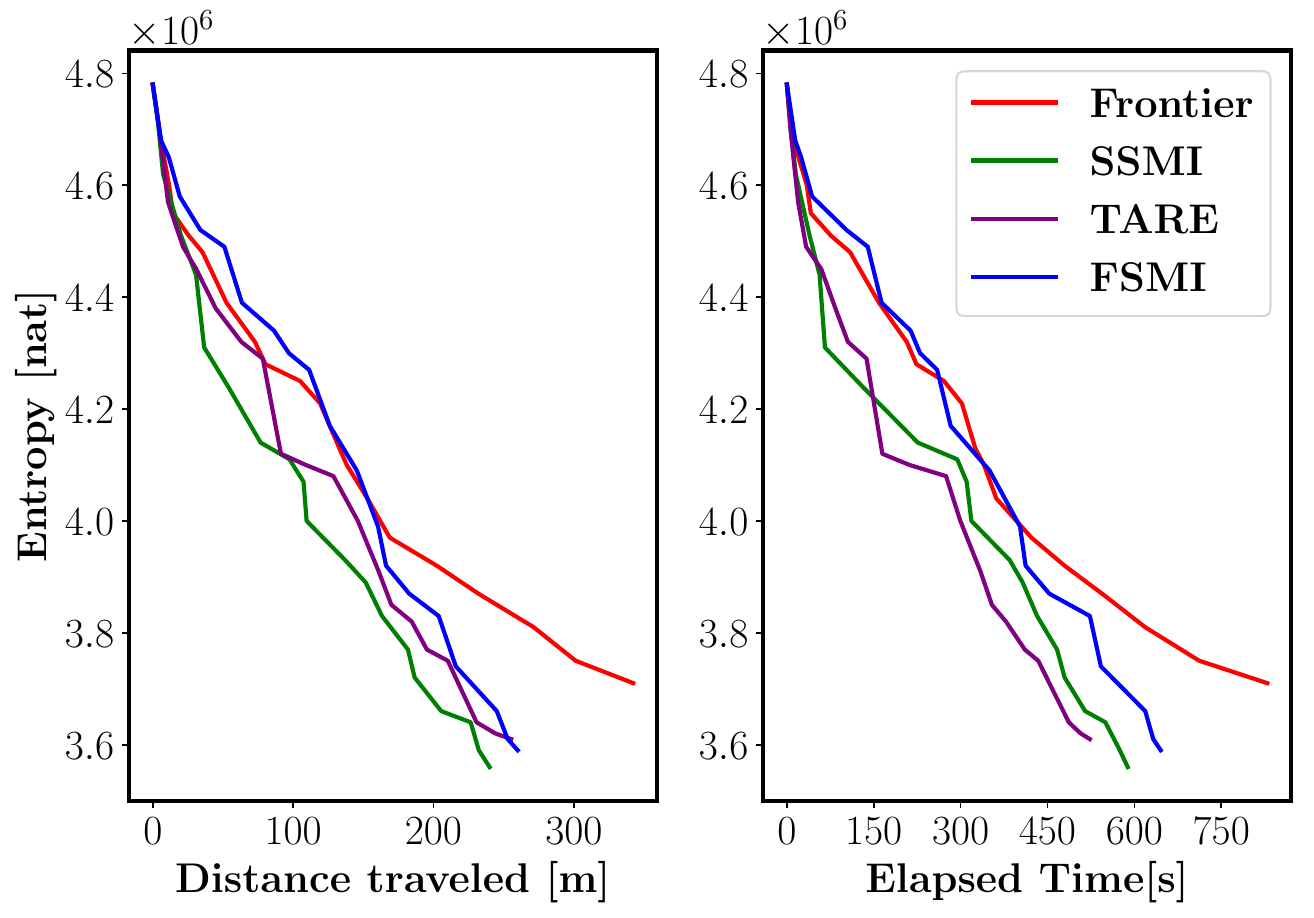}
  \caption{Simulation results for exploration in Unity 3-D environment.}
\label{fig:3d_sim_res}
\end{figure}

\begin{figure}[t]
  \centering
  \includegraphics[width=\linewidth]{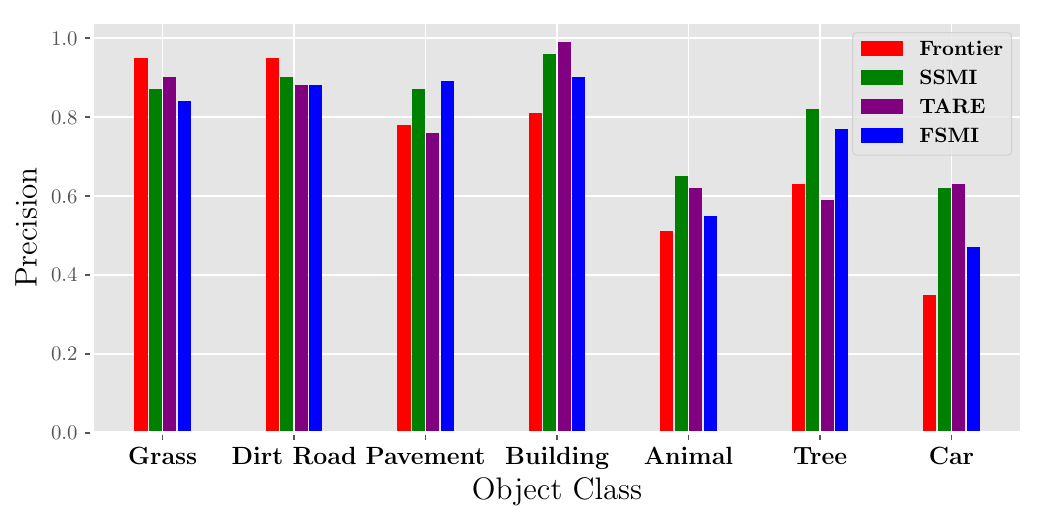}
  \caption{Mapping precision for observed semantic classes.}
\label{fig:3d_sim_res_prec}
\end{figure}

\subsection{3-D Mapping in a Real-world Outdoor Environment}
\label{subsec:real_world_map}

\begin{figure}[t]
    \centering
    \includegraphics[width=\linewidth]{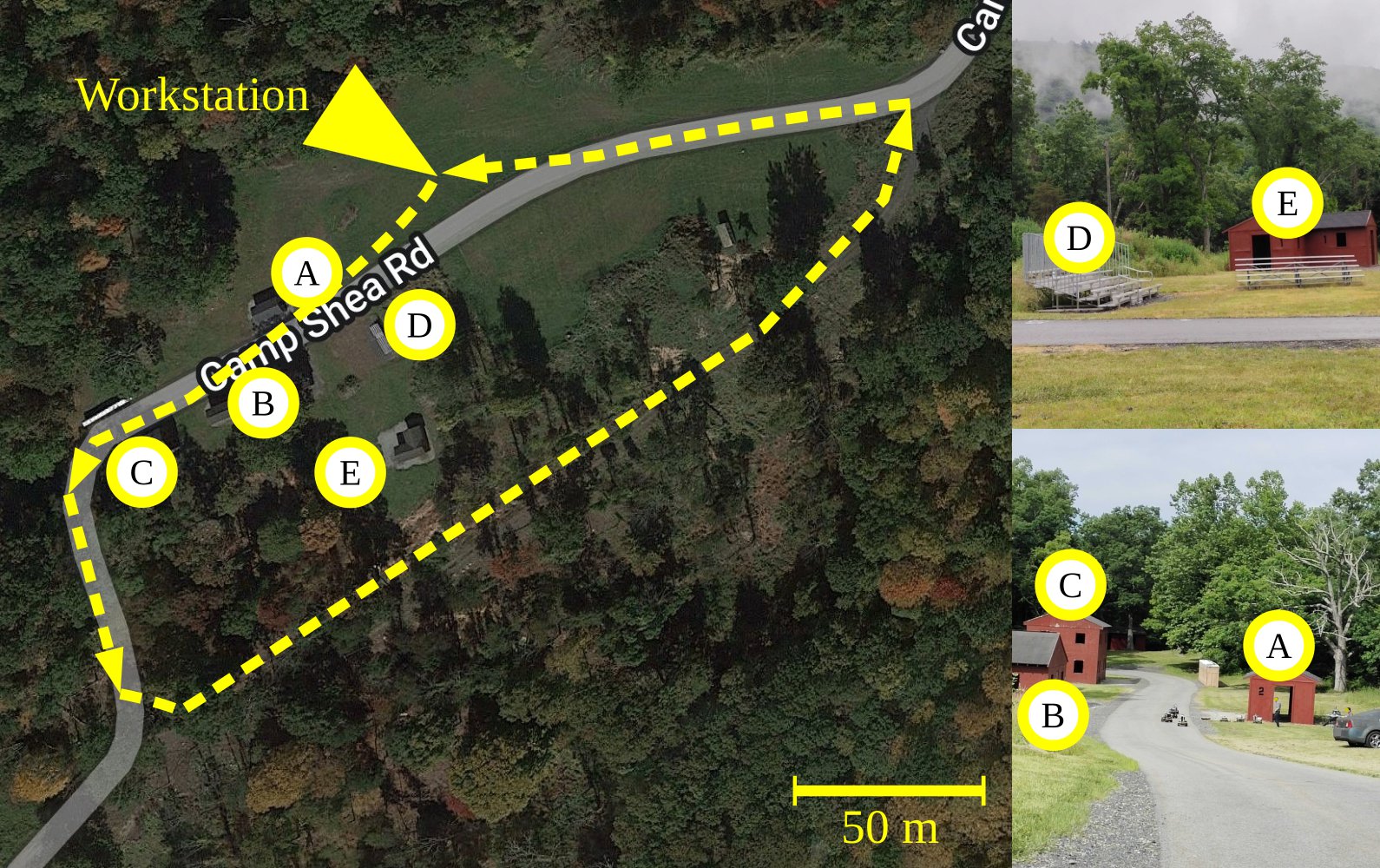}
    \caption{Environment for the outdoor mapping experiment. Left: Satellite image of the experiment locale with robot trajectory shown in yellow. Right: Corresponding locations from the ground level point of view.}
    \label{fig:3d_wp}
\end{figure}

We deployed our semantic OcTree mapping approach on a Husky robot equipped with an Ouster OS1-32 LiDAR and an Intel RealSense D455 RGBD camera. Our software stack is implemented using the \textit{Robot Operating System} (ROS) \cite{ros}. The LiDAR is used for localization via iterative closest point (ICP) scan matching \cite{icp}. A neural network based on a \textit{FCHarDNet} architecture \cite{fchardnet} and trained on the RUGD dataset \cite{rugd} was used for semantic segmentation. The RGBD camera produces color and depth images with size $640 \times 480$ at $30$ frames per second. The semantic segmentation algorithm takes a 2-D color image and outputs a semantic label for each pixel in the image, at an average frame rate of $28.7$ frames per second. By aligning the semantic image and the depth map, we derive a semantic 3-D point cloud which is utilized for Bayesian multi-class mapping. Our implementation was able to update the semantic OctoMap every $0.12$ s, on average, while all of the computations where performed on the mobile robot. The experiment was carried out in an approximately 6 acre forested area shown in Fig.~\ref{fig:3d_wp}. The environment contained various terrain features, including asphalt road, gravel, grass, densely forested areas, and hills. Additionally, a number of buildings and other structures such as bleachers, tents, and cars add to the diversity of the type of object categories within the locale. The robot was manually controlled via joystick, and traveled the path shown in Fig.~\ref{fig:3d_wp} (left) while incrementally building the semantic OctoMap. Fig.~\ref{fig:3d_wp_res} shows the semantic mapping result overlaying the satellite image obtained via 2-D projection of the semantic OctoMap. We computed the memory size of the semantic OctoMap, and compared it with the corresponding regular voxel grid representation, where each voxel contains the same amount of data as an OcTree leaf node at the lowest depth. Fig.~\ref{fig:3d_wp_size} shows an almost five-fold saving in memory when using OcTree data structure. The importance of the memory savings of the OctoMap representation becomes more apparent when communication is considered. Our semantic OctoMap implementation resulted in a network bandwidth requirement of $238$ KB/s for OctoMap, whereas a regular grid required $1173$ KB/s for map communication.

\begin{figure}[t]
    \centering
    \includegraphics[width=\linewidth]{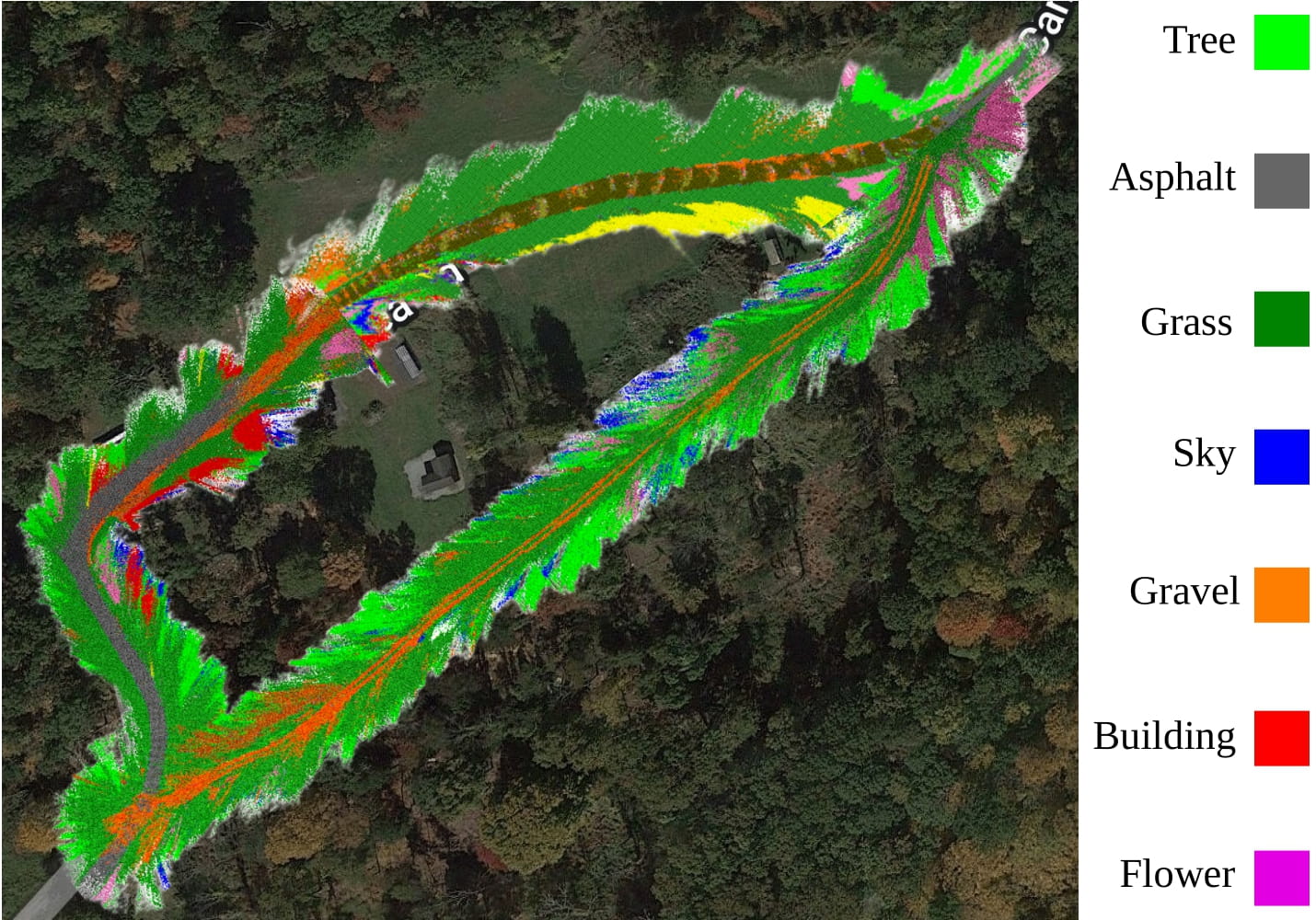}
    \caption{Semantic mapping output overlaying the satellite image. The map is obtained via 2-D projection of the 3-D semantic OctoMap.}
    \label{fig:3d_wp_res}
\end{figure}

\begin{figure}[t]
    \centering
    \includegraphics[width=\linewidth]{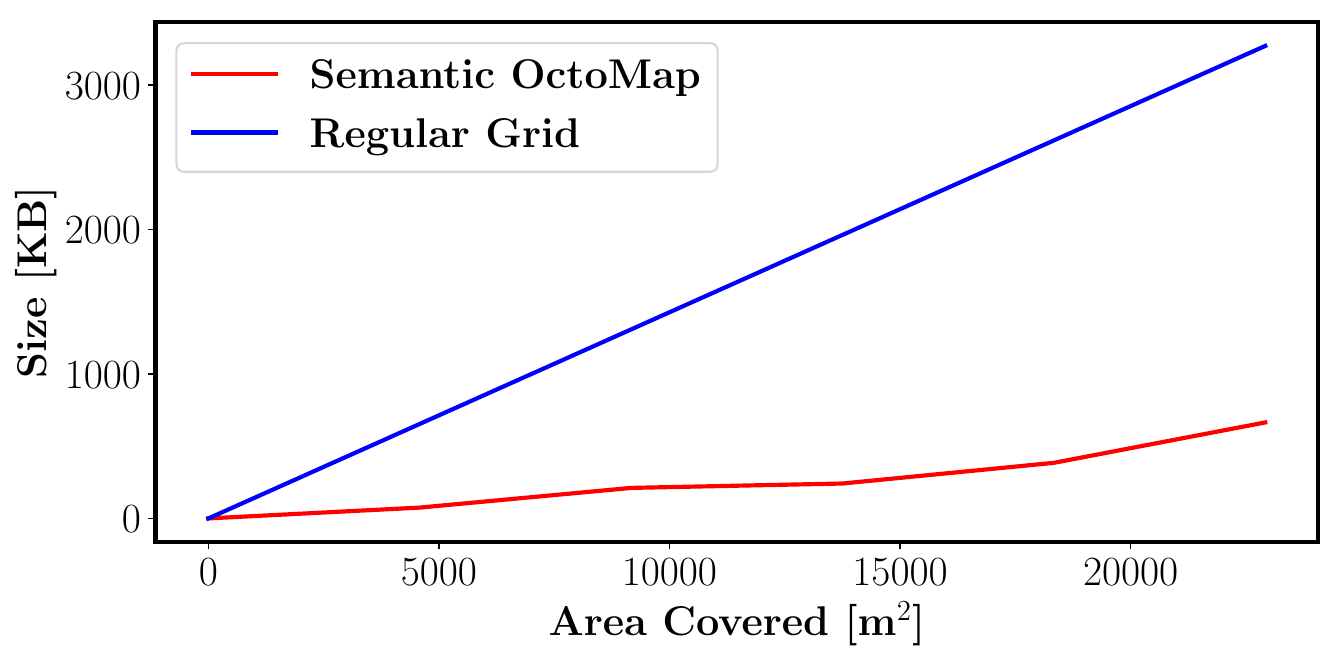}
    \caption{Memory use of regular grid vs. semantic OctoMap}
    \label{fig:3d_wp_size}
\end{figure}

\subsection{3-D Exploration in a Real-world Office Environment}
\label{subsec:real_world_exp}

\begin{figure}[t]
  \centering
  \includegraphics[width=0.8\linewidth]{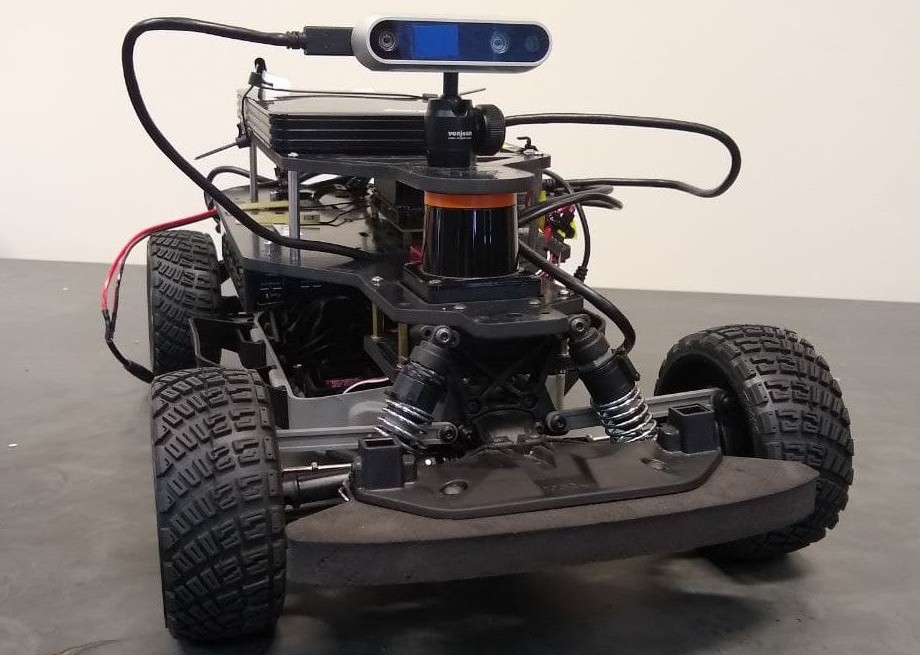}
  \caption{Robot car used in indoor real-world experiments.}
\label{fig:racecar}
\end{figure}

\begin{figure}[t]
  \centering
  \includegraphics[width=\linewidth]{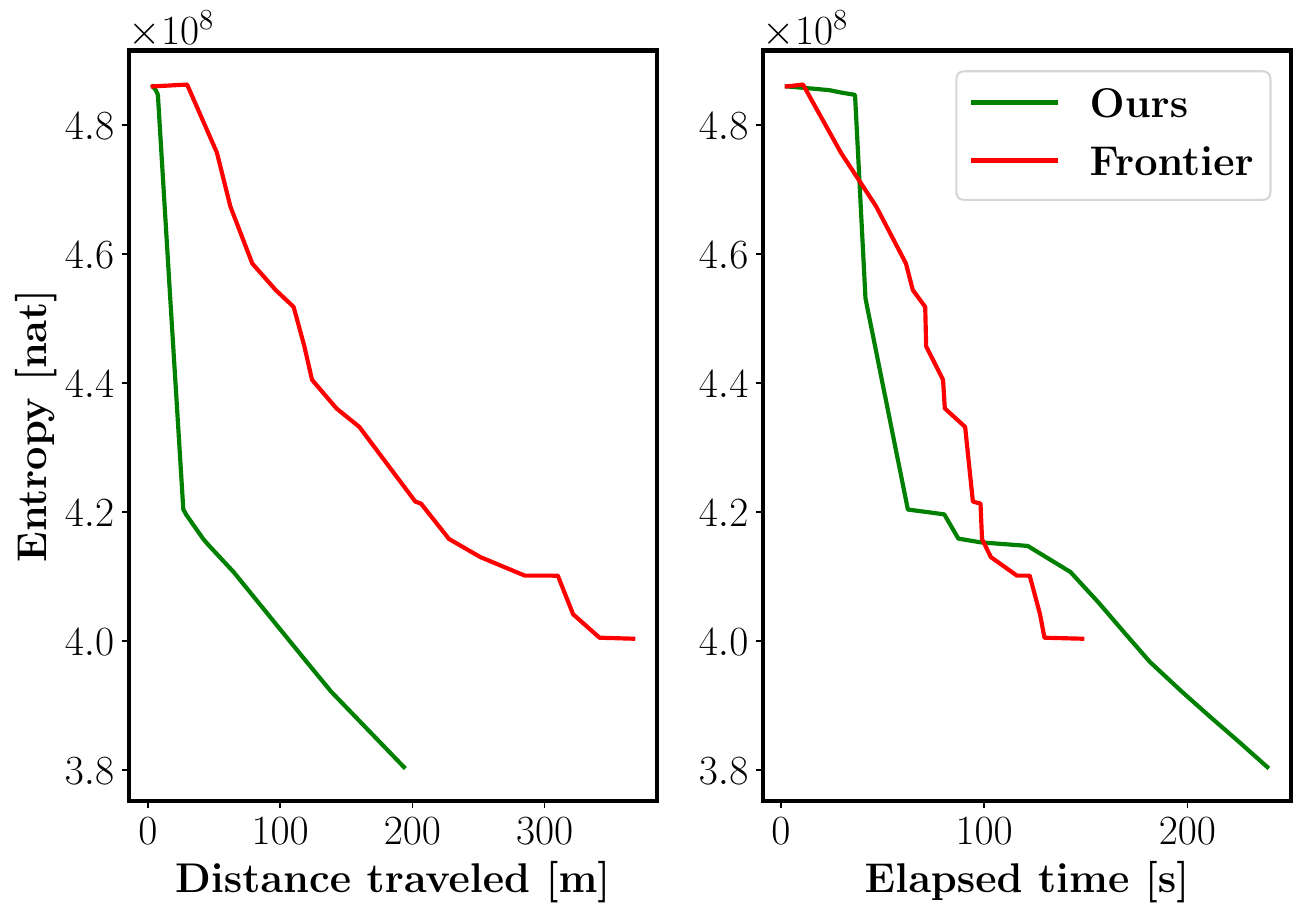}
  \caption{Real-world experiment results for active mapping for $20$ exploration iterations.}
\label{fig:3d_real_res}
\end{figure}

We implemented SSMI on a ground wheeled robot to autonomously map an indoor office environment. Fig.~\ref{fig:racecar} shows the robot equipped with a NVIDIA Xavier NX computer, a Hokuyo UST-10LX LiDAR, and an Intel RealSense D435i RGBD camera. Similar to the outdoor experiments, ROS was used for software deployment on the robot, and ICP laser scan matching provided localization. This time, we utilized a \textit{ResNet18}~\cite{resnet} neural network architecture pre-trained on the SUN RGB-D dataset \cite{sun_dataset} for semantic segmentation. In particular, we employed the deep learning inference ROS nodes provided by NVIDIA \cite{jetson}, which are optimized for Xavier NX computers via TensorRT acceleration. Due to limited computational power available on the mobile platform, we operated the RGBD camera at a lower frame rate of $15$ Hz with color and depth image size set to $640 \times 480$. The semantic segmentation algorithm was able to produce pixel classification images (resized to $512 \times 400$) at an average rate of $9.8$ frames per second. Our implementation was able to publish semantic OctoMap ROS topics every $0.34 s$, on average, with all of the processing occurred on the mobile platform. Fig.~\ref{fig:3d_real_env} depicts the exploration process, while Fig.~\ref{fig:3d_real_res} shows the performance of SSMI compared to frontier-based over $20$ exploration iterations. We observe that, similar to the simulations, SSMI outperforms frontier-based exploration in terms of distance traveled. Also, SSMI shows on par performance compared to Frontier in terms of entropy reduction per time. This can be explained by the fact that large depth measurement noise and classification error in the real-world experiments result in (a) the need for re-visiting explored areas in order to estimate an accurate map, leading to poor entropy reduction for the frontier-based method and (b) a small number of safe candidate trajectories, leading to fewer computations to be performed by SSMI. Overall, our experiments show that SSMI outperforms Frontier in indoor exploration scenarios where the number, and length, of candidate trajectories is constrained by the size of the environment.

\begin{figure*}[t]
    \begin{subfigure}[t]{0.49\linewidth}
    \centering
    \includegraphics[width=0.7\linewidth]{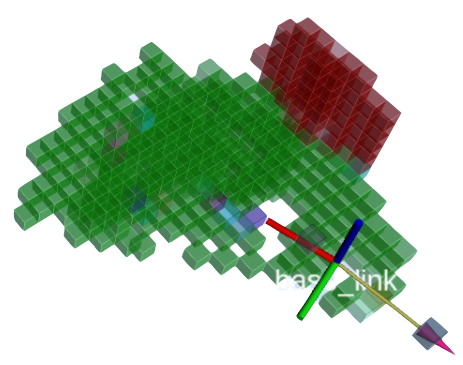}
    \captionsetup{justification=centering}
    \caption{The robot begins exploration.}
    \end{subfigure}%
    \hfill%
    \begin{subfigure}[t]{0.49\linewidth}
    \centering
    \includegraphics[width=0.7\linewidth]{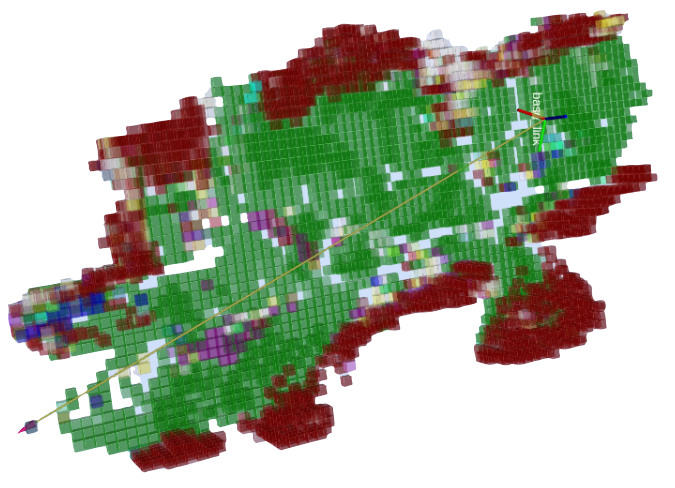}
    \caption{The robot visits neighboring unexplored regions while trying to refine the map of visited areas.}
    \end{subfigure}\\
    \begin{subfigure}[t]{0.49\linewidth}
    \centering
    \raisebox{0.12\textwidth}{\includegraphics[width=\linewidth]{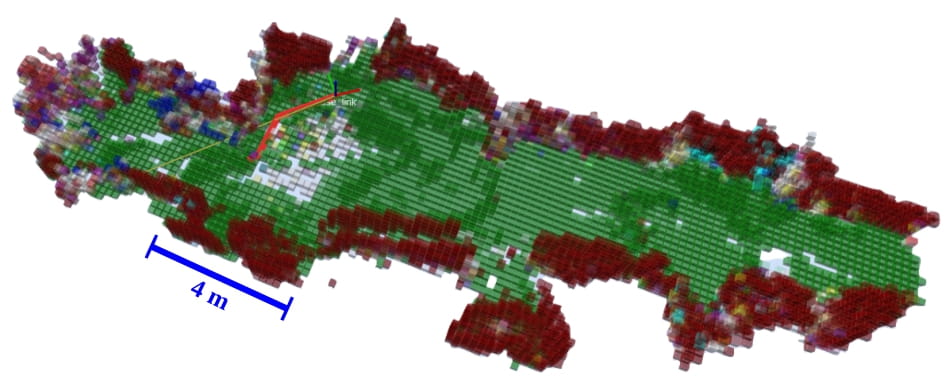}}
    \captionsetup{justification=centering}
    \caption{Semantic OctoMap after $20$ exploration iterations.}
    \end{subfigure}%
    \hfill%
    \begin{subfigure}[t]{0.49\linewidth}
    \centering
    \includegraphics[width=\linewidth]{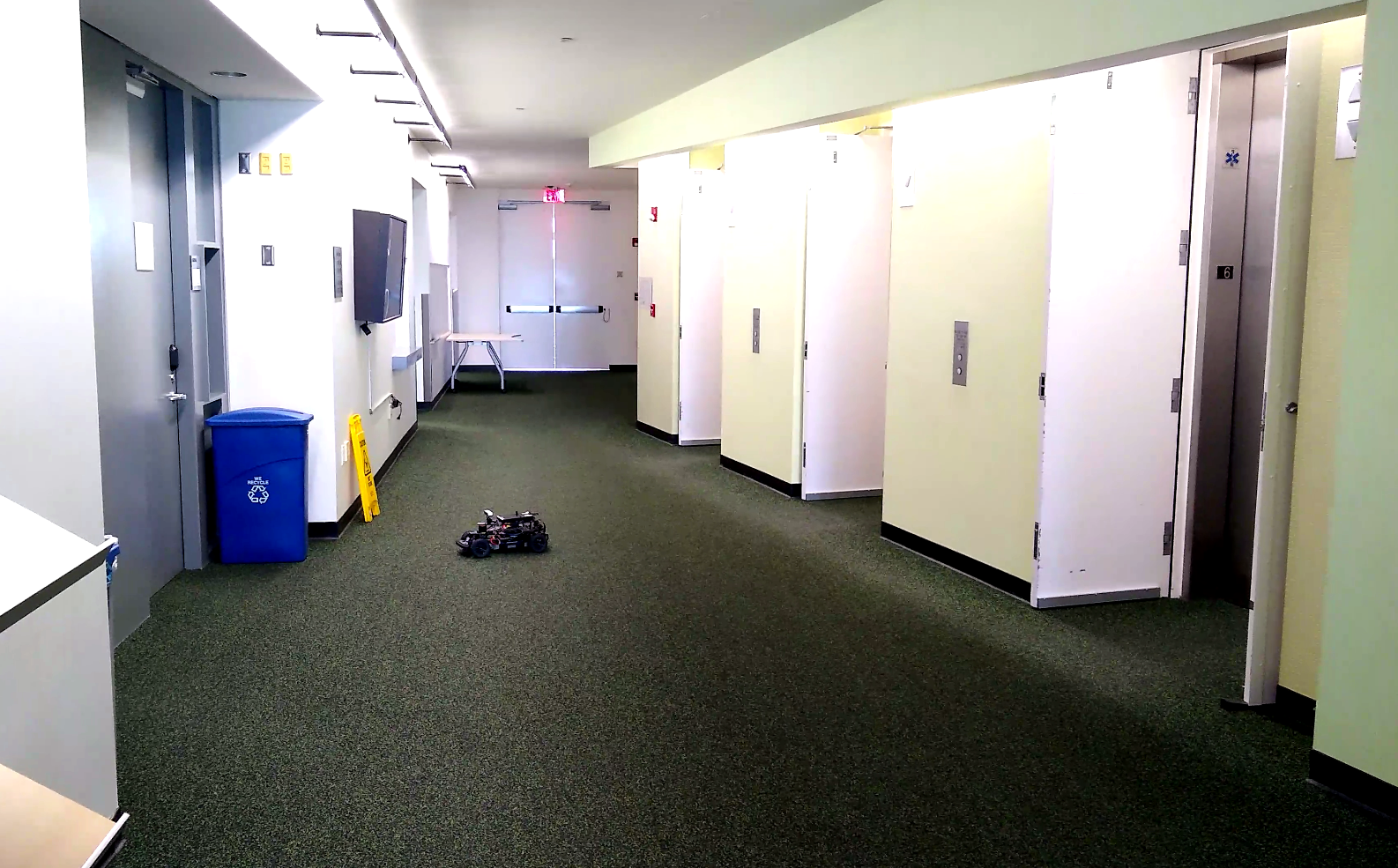}
    \captionsetup{justification=centering}
    \caption{Office environment featuring corridors, furniture, signs, and doors.}
    \end{subfigure}
    \caption{Time lapse of autonomous exploration and multi-class mapping in the environment shown in (d). The exploration is run for $20$ iterations. Different colors represent different semantic categories (floor, wall, furniture, etc.).}
    \label{fig:3d_real_env}
\end{figure*}

%% file: tex/conclusion.tex
\section{Conclusion}
\label{sec:conclusion}

This paper developed techniques for active multi-class mapping of large 3-D environments using range and semantic segmentation observations. Our results enable efficient mutual information computation over multi-class maps and make it possible to optimize for per-class uncertainty. Our experiments show that SSMI performs on par with the state of the art FSMI method in binary active mapping scenarios. However, when semantic information is considered SSMI outperforms existing algorithms and leads to efficient exploration and accurate multi-class mapping even in the presence of domain shift due to the difference between the classification training data and the testing environment. Experiments in both simulated and real-world environments showed the scalability of SSMI for large-scale 3-D exploration scenarios.

%% file: tex/appendix.tex
\section{Proof of Proposition~\ref{prop:log-odds-bayes-rule}}
\label{app:log-odds-bayes-rule}

Applying Bayes rule in \eqref{eq:bayes_rule} and the factorization in \eqref{eq:pdf_factorization} to $p_t(\bfm)$ for some $\bfz \in \calZ_{t+1}$ leads to:
\begin{equation}
\scaleMathLine[0.89]{\prod_{i=1}^N p(m_i | \calZ_{1:t}, \bfz) = \frac{p(\bfz)}{p(\bfz | \calZ_{1:t})} \prod_{i=1}^N \frac{p(m_i | \bfz)}{p(m_i)} p(m_i | \calZ_{1:t}).}
\end{equation}
The term $\frac{p(\bfz)}{p(\bfz | \calZ_{1:t})}$ may be eliminated by considering the odds ratio of an arbitrary category $m_i = k_i \in \calK$ versus the free category $m_i = 0$ for each cell $i$:
\begin{align}
\prod_{i=1}^N &\frac{ p(m_i = k_i | \calZ_{1:t}, \bfz) }{ p(m_i = 0 | \calZ_{1:t}, \bfz) } \\
&= \prod_{i=1}^N \frac{p(m_i = k_i | \bfz)}{p(m_i = 0 | \bfz)} \frac{p(m_i = 0)}{p(m_i = k_i)} \frac{p(m_i = k_i | \calZ_{1:t})}{p(m_i = 0| \calZ_{1:t})}.\notag
\end{align}
Since each term in both the left- and right-hand side products only depends on one map cell $m_i$, the expression holds for each individual cell. Re-writing the expression for cell $m_i$ in vector form, with elements corresponding to each possible value of $k_i \in \calK$, and taking an element-wise log leads to:
\begin{equation}
\label{eq:13}
\begin{aligned}
\begin{bmatrix} \log \frac{p(m_i = 0 | \calZ_{1:t},\bfz)}{p(m_i = 0| \calZ_{1:t},\bfz)} & \cdots & \log \frac{p(m_i = K| \calZ_{1:t},\bfz)}{p(m_i = 0| \calZ_{1:t},\bfz)} \end{bmatrix}^\top\\
\qquad = (\bfl_i(\bfz) - \bfh_{0,i}) + \bfh_{t,i}.
\end{aligned}
\end{equation}
Applying \eqref{eq:13} recursively for each element $\bfz \in \calZ_{t+1}$ leads to the desired result in \eqref{eq:log-odds-bayes-rule}.\qed

\section{Proof of Proposition~\ref{prop:mut_inf_semantic}}
\label{app:mut-inf-semantic}

Let $\calR_{t+1:t+T}(r_{max}) := \cup_{\tau,b} \calR_{\tau,b}(r_{max})$ be the set of map indices which can potentially be observed by $\underline{\calZ}_{t+1:t+T}$. Using the factorization in \eqref{eq:pdf_factorization} and the fact that Shannon entropy is additive for mutually independent random variables, the mutual information only depends on the cells whose index belongs to $\calR_{t+1:t+T}(r_{max})$, i.e.:
\begin{align}
    I(\bfm&; \underline{\calZ}_{t+1:t+T}  \mid \calZ_{1:t})\notag\\
    &= \sum_{\tau = t+1}^{t+T} \sum_{b = 1}^B \sum_{i \in \calR_{\tau,b}(r_{max})} I(m_i; \bfz_{\tau,b} \mid \calZ_{1:t}). \label{eq:mut_inf_decomp}
\end{align}
This is true because the measurements $\bfz_{\tau,b} \in \underline{\calZ}_{t+1:t+T}$ are independent by construction and the terms $I(m_i; \underline{\calZ}_{t+1:t+T} \mid \calZ_{1:t})$ can be decomposed into sums of mutual information terms between single-beam measurements $\bfz_{\tau,b}$ and the respective observed map cells $m_i$. The mutual information between a single map cell $m_i$ and a sensor ray $\bfz$ is:
\begin{align}
  &I(m_i; \bfz \mid \calZ_{1:t}) = \label{eq:mut_inf_single_cell}\\
  &\scaleMathLine{\int p(\bfz \mid \calZ_{1:t}) \sum_{k=0}^K p(m_i = k \mid \bfz, \calZ_{1:t}) \log{\frac{p(m_i = k \mid \bfz, \calZ_{1:t})}{p_t(m_i = k)}} d\bfz.} \notag
\end{align}
Using the inverse observation model in \eqref{eq:log_inverse_observation_model} and the Bayesian multi-class update in \eqref{eq:log-odds-bayes-rule}, we have:
\begin{align}
    &\sum_{k=0}^K p(m_i = k \mid \bfz, \calZ_{1:t}) \log{\frac{p(m_i = k \mid \bfz, \calZ_{1:t})}{p_t(m_i = k)}} \notag\\ 
    &\scaleMathLine{ = (\bfl_i(\bfz) - \bfh_{0,i})^\top \sigma(\bfl_i(\bfz) - \bfh_{0,i} + \bfh_{t,i} ) +\log{\frac{p(m_i = 0 \mid \bfz, \calZ_{1:t})}{p_t(m_i = 0)}}}\notag\\
    & = f(\bfl_i(\bfz) - \bfh_{0,i}, \bfh_{t,i}), \label{eq:f_func}
\end{align}
where \eqref{eq:log_inverse_observation_model} and \eqref{eq:log-odds-bayes-rule} were applied a second time to the log term above. Plugging \eqref{eq:f_func} back into the mutual information expression in \eqref{eq:mut_inf_single_cell} and returning to \eqref{eq:mut_inf_decomp}, we have:
\begin{align}
\label{eq:mi-integral}
I(&\bfm; \underline{\calZ}_{t+1:t+T}  \mid \calZ_{1:t})\\
&= \sum_{\tau = t+1}^{t+T} \sum_{b = 1}^B \sum_{y = 1}^K \int_0^{r_{max}} \biggl( p(\bfz_{\tau,b} = (r,y) \mid \calZ_{1:t}) \notag\\
&\qquad\qquad\qquad \sum_{i \in \calR_{\tau,b}(r_{max})} \negquad f(\bfl_i((r,y)) - \bfh_{0,i}, \bfh_{t,i}) \biggr) dr. \notag
\end{align}
For $\bfz_{\tau,b}=(r,y)$, the second term inside the integral above can be simplified to:
\begin{equation}
\label{eq:C_cal}
\begin{aligned}
\Tilde{C}_{\tau,b}(r,y) &:= \negquad \sum_{i \in \calR_{\tau,b}(r_{max})} \negquad f(\bfl_i((r,y)) - \bfh_{0,i}, \bfh_{t,i})\\
&\phantom{:}= f(\bfphi^+ + \bfE_{y+1}\bfpsi^+ - \bfh_{0,i_{\tau,b}^*}, \bfh_{t,i_{\tau,b}^*})\\
&\qquad + \negquad\sum_{i \in \calR_{\tau,b}(r) \setminus \{i_{\tau,b}^*\}} \negquad f(\bfphi^- - \bfh_{0,i}, \bfh_{t,i})
\end{aligned}
\end{equation}
because for map indices $i \in \calR_{\tau,b}(r_{max}) \setminus \calR_{\tau,b}(r)$ that are not observed by $\bfz_{\tau,b}$, we have $\bfl_i((r,y)) = \bfh_{0,i}$ according to \eqref{eq:log_inverse_observation_model} and $f(\bfh_{0,i} - \bfh_{0,i}, \bfh_{t,i}) = 0$.

Next, we apply the definition of \eqref{eq:cond_prob_approx} for the first term in the integral in \eqref{eq:mi-integral}, which turns it into an integration over $\tilde{p}_{\tau,b}(r,y) \tilde{C}_{\tau,b}(r,y)$. Note that $\Tilde{p}_{\tau,b}(r,y)$ and $\Tilde{C}_{\tau,b}(r,y)$ are piecewise-constant functions since $\calR_{\tau,b}(r)$ is constant with respect to $r$ as long as the beam $\bfz$ lands in cell $m_{i^*}$. Hence, we can partition the integration domain over $r$ into a union of intervals where the beam $\bfz$ hits the same cell, i.e., $\calR_{\tau,b}(r)$ remains constant:
\begin{equation*}
\int_{0}^{r_{max}} \negquad\Tilde{p}_{\tau,b}(r,y) \Tilde{C}_{\tau,b}(r,y)\,dr = \sum_{n=1}^{N_{\tau,b}} \int_{r_{n-1}}^{r_n} \negquad\Tilde{p}_{\tau,b}(r,y) \Tilde{C}_{\tau,b}(r,y)\,dr,
\end{equation*}
where $N_{\tau,b} = |\calR_{\tau,b}(r_{max})|$, $r_0 = 0$, and $r_N = r_{max}$. From the piecewise-constant property of $\Tilde{p}_{\tau,b}(r,y)$ and $\Tilde{C}_{\tau,b}(r,y)$ over the interval $(r_{n-1},r_n]$, one can obtain:
\begin{align}
    \int_{r_{n-1}}^{r_n} &\Tilde{p}_{\tau,b}(r,y) \Tilde{C}_{\tau,b}(r,y)\,dr \\ 
    &= \Tilde{p}_{\tau,b}(r_n,y) \Tilde{C}_{\tau,b}(r_n,y) \gamma(n) = p_{\tau,b}(n,y) C_{\tau,b}(n,y), \notag
\end{align}
where $p_{\tau,b}(n,y)$ and $C_{\tau,b}(n,y)$ are defined in the statement of Proposition~\ref{prop:mut_inf_semantic}. Substituting $y$ with $k$ and plugging the integration result into \eqref{eq:mi-integral} yields the lower bound in \eqref{eq:mut_inf_semantic} for the mutual information between $\bfm$ and $\calZ_{t+1:t+T}$.\qed

\section{Proof of Proposition~\ref{prop:mut_inf_semantic_octomap}}
\label{app:mut-inf-semantic-octomap}

Consider a single beam $\bfz_{\tau, b}$, passing through cells $\left\{m_i\right\}_i$, $i \in \calR_{\tau, b}(r_{max})$. As shown in Appendix~\ref{app:mut-inf-semantic}, the mutual information between the map $\bfm$ and a beam $\bfz_{\tau, b}$ can be computed as:
\begin{equation}
\label{eq:mut-inf-single-ray}
    I(\bfm; \bfz_{\tau, b} | \calZ_{1:t}) = \sum_{k=1}^{K} \sum_{n=1}^{N_{\tau, b}} p_{\tau, b}(n,k) C_{\tau, b}(n,k).
\end{equation}
Assuming piece-wise constant class probabilities, we have:
\begin{equation}
\begin{split}
    \sum_{n=1}^{N_{\tau, b}} p_{\tau, b}(n,k) &C_{\tau, b}(n,k) = \\ \sum_{q=1}^{Q_{\tau, b}} &\sum_{n= \omega_{\tau, b, 1:q-1} + 1}^{\omega_{\tau, b, 1:q}} p_{\tau, b}(n,k) C_{\tau, b}(n,k),
\end{split}
\label{eq:piece_wise_const_info}
\end{equation}
where $\omega_{\tau, b, 1:q} = \sum_{j=1}^q \omega_{\tau, b, j}$. For each $\omega_{\tau, b, 1:q-1} < n \leq \omega_{\tau, b, 1:q}$, the terms $p_{\tau, b}(n,k)$ and $C_{\tau, b}(n,k)$ are expressed as:
\begin{equation}
\label{eq:p_and_C_const}
\begin{aligned}
    p_{\tau, b}(n,k) &= \pi_t(q,k) \pi_t^{(n-1-\omega_{\tau, b, 1:q-1})}(q,0) \prod_{j=1}^{q-1} \pi_t^{\omega_{\tau, b, j}}(j,0),\\
    C_{\tau, b}(n,k) &= f(\bfphi^+ + \bfE_{k+1} \bfpsi^+ - \bfchi_{0,q}, \bfchi_{t,q}) + \\ &(n-1-\omega_{\tau, b, 1:q-1}) f(\bfphi^--\bfchi_{0,q}, \bfchi_{t,q}) +\\ &\qquad\qquad\qquad\sum_{j=1}^{q-1} \omega_{\tau, b, j} f(\bfphi^--\bfchi_{0,j}, \bfchi_{t,j}).\notag
\end{aligned}
\end{equation}
Plugging this into the inner summation of \eqref{eq:piece_wise_const_info} leads to:
\begin{equation}
\begin{split}
    \sum_{n= \omega_{\tau, b, 1:q-1} + 1}^{\omega_{\tau, b, 1:q}} &p_{\tau, b}(n,k) C_{\tau, b}(n,k) = \\ \rho_{\tau, b}&(q,k) \big[\beta_{\tau, b}(q,k) \sum_{j = 0}^{\omega_{\tau, b, q} - 1} \pi_t^{j}(q,0) + \\ &f(\bfphi^--\bfchi_{0,q}, \bfchi_{t,q}) \sum_{j = 0}^{\omega_{\tau, b, q} - 1} j \pi_t^j(q,0) \big],
\end{split}
\label{eq:within_section_sum}
\end{equation}
The summations in \eqref{eq:within_section_sum} can be computed explicitly, leading to the following closed-form expression:
\begin{align} \label{eq:within_section_sum_res}
    &\beta_{\tau, b}(q,k) \sum_{j = 0}^{\omega_{\tau, b, q} - 1} \pi_t^{j}(q,0) + \notag\\ 
    &\qquad f(\bfphi^--\bfchi_{0,q}, \bfchi_{t,q}) \sum_{j = 0}^{\omega_{\tau, b, q} - 1} j \pi_t^j(q,0) = \notag\\ 
    &\beta_{\tau, b}(q,k) \frac{1 - \pi_t^{\omega_{\tau,b,q}}(q,0)}{1 - \pi_t(q,0)} + \\ 
    &\qquad\frac{f(\bfphi^--\bfchi_{0,q}, \bfchi_{t,q})}{(1 - \pi_t(q,0))^2} \Big[(\omega_{\tau,b,q} - 1) \pi_t^{\omega_{\tau,b,q} + 1}(q,0) \notag\\ 
    &\qquad\;\;\:- \omega_{\tau,b,q} \pi_t^{\omega_{\tau,b,q}}(q,0) + \pi_t(q,0)\Big] = \Theta_{\tau, b}(q,k).\notag
\end{align}
Therefore, the Shannon mutual information between a semantic OctoMap $\bfm$ and a range-category measurement $\bfz_{\tau, b}$ can be computed as in \eqref{eq:mut_inf_semantic_octomap_single}.\qed